\documentclass{article}
\usepackage{ijcai24}

\usepackage{times}
\usepackage{soul}
\usepackage{url}
\usepackage[hidelinks]{hyperref}
\usepackage[utf8]{inputenc}
\usepackage[small]{caption}
\usepackage{graphicx}
\usepackage{amsmath}
\usepackage{amsthm}
\usepackage{booktabs}
\usepackage[switch]{lineno}
\usepackage{thmtools,thm-restate}
\usepackage{aligned-overset}

\usepackage{amssymb}
\usepackage{amsfonts}
\usepackage{mathtools}
\usepackage{dsfont} % Used for \N \R etc.
\usepackage{xspace} % For controlling spaces after new command
\usepackage{enumitem} % Used to control spaces in itemize
\usepackage{subcaption} % ensure subfigure will work

\theoremstyle{plain} % this sets the style for all new environments created using \newtheorem to have the "plain" style, which as a bold title, italic text and vertical space above and below it. 
\newtheorem{theorem}{Theorem}
\newtheorem*{theorem*}{Theorem}
\newtheorem{lemma}{Lemma}

\theoremstyle{definition} % this sets the style for all new environments created using \newtheorem to have the "definition" style, which as a bold title, upright text and vertical space above and below it.
\newtheorem{defn}{Definition}

\theoremstyle{remark} % this sets the style for all new environments created using \newtheorem to have the "remark" style, which as an italic non-bold title, upright text and no extra vertical space above and below it.
 
\newcounter{casecount}

\AtBeginEnvironment{proof}{\setcounter{casecount}{0}}
\usepackage{romannum}

\renewcommand{\epsilon}{\varepsilon}
\newcommand{\N}{\ensuremath{\mathds{N}}\xspace}
\newcommand{\R}{\ensuremath{\mathds{R}}\xspace}

\newcommand{\X}{\ensuremath{\mathcal{X}}\xspace}
\newcommand{\Y}{\ensuremath{\mathcal{Y}}\xspace}

\newcommand{\F}{\ensuremath{\mathcal{F}}\xspace}
\DeclareMathOperator{\OPT}{OPT}
\newcommand{\CHALLENGE}{\textsc{Challenge}\xspace}
\newcommand{\Challenge}{\CHALLENGE}

\newcommand{\floor}[1]{\ensuremath{\left\lfloor#1\right\rfloor}}

\newcommand{\abs}[1]{\ensuremath{\left\lvert#1\right\rvert}}

\newcommand{\E}[1]{\ensuremath{\mathrm{E}\mathord{\left(#1\right)}}}
\newcommand{\Et}[1]{\ensuremath{\mathrm{E}_t\mathord{\left(#1\right)}}}

\newcommand{\Ek}[1]{\ensuremath{\mathrm{E}_k\mathord{\left(#1\right)}}}

\newcommand{\Prob}[1]{\ensuremath{\mathrm{Pr}\left(#1\right)}}
\newcommand{\prob}[1]{\Prob{#1}}

\newcommand{\ones}[1]{\left\lvert#1\right\rvert_1}

\newcommand{\BILINEAR}{\textsc{Bilinear}\xspace}
\newcommand{\bilinear}{\BILINEAR}
\newcommand{\RWAB}{\textsc{Rwab}\xspace}
\newcommand{\rwab}{\RWAB}
\newcommand{\RN}[1]{%
  \textup{\uppercase\expandafter{\romannumeral#1}}%
}

\newcommand{\pminus}{\ensuremath{p^-_{x,y}}\xspace}
\newcommand{\pplus}{\ensuremath{p^+_{x,y}}\xspace}

\newcommand{\M}{M\xspace}

\usepackage{algorithm}
\usepackage[noend]{algpseudocode}

\title{Concentration Tail-Bound Analysis of
Coevolutionary\\
and Bandit Learning Algorithms\footnote{
Full version at \url{https://arxiv.org/abs/2405.04480}. }
}

\author{
Per Kristian Lehre
\and
Shishen Lin~\footnote{Authors are listed in alphabetical order.}
\affiliations
University of Birmingham, Birmingham, United Kingdom
\emails
\{p.k.lehre, sxl1242\}@cs.bham.ac.uk
}

\begin{document}
\maketitle

\begin{abstract}
Runtime analysis, as a branch of the theory of AI, studies how the number of iterations algorithms take before finding a solution (its runtime) depends on the design of the algorithm and the problem structure. Drift analysis is a state-of-the-art tool for estimating the runtime of randomised algorithms, such as evolutionary and bandit algorithms. Drift refers roughly to the expected progress towards the optimum per iteration. This paper considers the problem of deriving concentration tail-bounds on the runtime/regret of algorithms. It provides a novel drift theorem that gives precise exponential tail-bounds given positive, weak, zero and even negative drift. Previously, such exponential tail bounds were missing in the case of weak, zero, or negative drift. 
\par Our drift theorem can be used to prove a strong concentration of the runtime/regret of algorithms in AI. For example, we prove that the regret of the \rwab bandit algorithm is highly concentrated, while previous analyses only considered the expected regret. This means that the algorithm obtains the optimum within a given time frame with high probability, i.e. a form of algorithm reliability. Moreover, our theorem implies that the time needed by the co-evolutionary algorithm RLS-PD to obtain a Nash equilibrium in a \bilinear max-min-benchmark problem is highly concentrated. However, we also prove that the algorithm forgets the Nash equilibrium, and the time until this occurs is highly concentrated. This highlights a weakness in the RLS-PD which should be addressed by future work.

\end{abstract}

\section{Introduction}
Drift analysis is a powerful technique in understanding the performance of randomised algorithms, particularly in the field of runtime analysis of heuristic search. 
For more recent overviews of drift analysis in evolutionary computation, see ~\cite{doerr2019theory,Intro_Neumann,Thomas_intro,he_drift_2001,hajek1982hitting}. 
The majority of existing drift theorems provide an upper bound on the expected runtime needed to reach a target state, such as an optimal solution set \cite{he_drift_2001}. By identifying an appropriate potential function and demonstrating a positive drift towards the target state, the expected runtime can be bounded by the reciprocal of the drift multiplied by the maximum distance from the target state. 
\par The focus of concentration tail-bound analysis is on quantifying the deviation of the runtime of randomised algorithm $T$, from its expected value. 
By providing insights into the distribution of $T$, this approach offers a more detailed understanding of an algorithm's performance~\cite{kotzing_concentration_2016,lehre_general_2018,doerr_adaptive_2013}. The concentration tail-bound analysis has gained significant interest due to its potential for delivering tighter upper bounds on the runtime of various algorithms. For instance, an exponential tail bound is used to bound the expected runtime of RLS on separable functions~\cite{SeparableDoerr}. Moreover, in the case of (1+1)-cooperative co-evolutionary algorithms, concentration tail-bound analysis can establish a $\Theta (n \log (n))$ bound for the runtime of the cooperative coevolutionary algorithm on linear functions~\cite{shishen2023}. 
The concentration tail-bound analysis is also useful in the context of restarting arguments; for example~\cite {Restart_Per}. More precise runtime estimation can be valuable in 
optimising and comparing different algorithms, potentially
leading to improved algorithm design and 
performance \cite{bian2020efficient,PK2021escaping,zheng2022first,doerr2023runtime}. 
\par Concentration tail bounds are not only used in runtime analysis to help us understand evolutionary algorithms, including simple genetic algorithms or coevolutionary algorithms~\cite{kotzing_concentration_2016,shishen2023},
but can also be used in regret analysis of reinforcement learning algorithms. 
A typical example is using concentration inequality (Azuma-Hoeffding inequality) to provide precise bounds for regret. 
Concentration inequalities are used in the development of optimal UCB family algorithms, incorporating the concept of optimism in the face of uncertainty~\cite{auer2002finite}.

\subsection{Related Works}
Researchers use drift analysis to analyse not only the runtime of evolutionary algorithms but also other randomised algorithms like Random 2-SAT \cite{gobel2018intuitive} or the expected regret of simple reinforcement learning algorithms on bandit problems \cite{larcher2023simple}. We would like to explore more advanced drift analysis tools to provide more precise estimates of runtime and regret.
Various extensions of drift theorems have been proved, including multiplicative drift~\cite{doerr_multiplicative_nodate}, variable drift~\cite{VariableBS,johannsen2010random,mitavskiy2009theoretical}, and negative drift~\cite{oliveto_erratum_2012}. The multiplicative drift theorem~\cite{doerr_multiplicative_nodate} refines the original additive drift theorem by considering the current state, resulting in a more precise bound when using the same potential function. Variable drift~\cite{VariableBS,johannsen2010random,mitavskiy2009theoretical} generalises the multiplicative drift concept to incorporate an increasing positive function, $h$. 
On the other hand, negative drift~\cite{oliveto_erratum_2012} is employed to provide a lower bound for the expected runtime, often used to demonstrate that an algorithm has an exponential expected runtime, thereby proving its inefficiency.

In recent years, researchers have been exploring advanced drift theorems that focus on the tail bound of the runtime~\cite{kotzing_concentration_2016,lehre_general_2018,doerr_adaptive_2013}. 
Researchers are also interested in the applications of concentration inequalities, like the Azuma-Hoeffding inequalities~\cite{Azuma1967WEIGHTEDSO}. They provide deeper insights into the behaviour and performance of randomised algorithms \cite{Benjamin_Azuma}.

\subsection{Our Contributions}
\par This paper provides a novel perspective on analysing tail bounds by introducing a classic recurrence strategy. 
With the help of the recurrence strategy, this paper presents a sharper bound for all possible drift cases with a simpler proof. 
In particular, we provide an exponential tail bound under constant variance with negative drift.  
Refining an existing method, we also show
a more precise exponential tail bound for the traditional cases with additive drift and for the cases in which there is constant variance but weak or zero drift. 
\par Finally, we illustrate the practical impact of our findings by applying our theorems to various algorithms. The analysis brings us stronger performance guarantees for these algorithms. 
In particular, we prove the instability of the co-evolutionary algorithm (CoEA) on maximin optimisation (\bilinear problem instance) occurs with high probability. 
Moreover, we show that the randomness of the reinforcement learning algorithm \rwab can help to find the optimal policy for the 2-armed non-stationary bandit problem with high probability. 
This paper is the first tail-bound analysis of both random local search with pairwise dominance (RLS-PD) and the bandit learning algorithm \rwab.

\section{Preliminaries}
For a filtration $\F_t$, we write $\Et{\cdot}:=\E{\cdot \mid \mathcal{F}_{t}}$.
We denote the $1$-norm as $|z|_1=\sum_{i=1}^n z_{i}$ for $z \in \{0,1\}^n$ and 
$\mathds{1}_{E}$ by indicator function, i.e. $\mathds{1}_{E}=1$ if event $E$ holds and $0$ otherwise.
With high probability" will be abbreviated as "w.h.p.". We say an event $E_n$ with problem size $n\in \mathbb{N}$ occurs w.h.p. if $\Pr (E_n) \geq 1-1/{\text{poly}(n)}$.
We defer pseudo-codes of algorithms and tables in the appendix.

We define the $k$-th stopping time, also called $k$-th hitting time, which will be used in later proofs. 
\begin{defn}($k$-th stopping time) Given a stochastic process $(X_t)_{t \geq 0}$ on a state space in $\mathbb{R}$. 
Let the target set $A$ be a finite non-empty subset of $\mathbb{R}$, and then for any $k \geq 0$, we define $T_{k} = \min \{t \geq k \mid X_{t} \in A \}$. 
In particular, $T_0$ is the first hitting time at $A$.
\end{defn}

We first provide a formal definition of variance-dominated and variance-transformed processes.
\begin{defn}(Variance-dominated processes) A sequence of random variables $X_0,X_1,\dots \in [0,n]$ is a variance-dominated process with respect to the filtration $\mathcal{F}_0, \mathcal{F}_1, \dots $ if for all $t \in \mathbb{N}$, the following conditions hold: 
\begin{itemize}
    \item[(1)] $\E{X_{t+1}-X_{t} \mid \mathcal{F}_{t}}\geq 0$;

    \item[(2)]  $\exists \delta>0$ such that $\E{(X_{t+1}-X_{t})^2 \mid \mathcal{F}_{t}}\geq \delta$.
\end{itemize}
\end{defn}

\begin{defn}(Variance-transformed processes) A sequence of random variables $X_0, X_1,\dots \in [0,n]$ is a variance-transformed process with respect to the filtration $\mathcal{F}_0, \mathcal{F}_1, \dots $ if for all $t \in \mathbb{N}$, the following conditions hold: 
\begin{itemize}
    \item[(1)] $0> \E{X_{t+1}-X_{t} \mid \mathcal{F}_{t}}\geq -\frac{c}{n}$;

    \item[(2)] $\exists \delta>0$ such that $\E{(X_{t+1}-X_{t})^2 \mid \mathcal{F}_{t}}\geq \delta$.
\end{itemize}
\end{defn}

\par This paper mainly focuses on 
random processes which consist of positive, weak (almost zero) 
or even a small negative drift with a constant second moment 
since these processes exhibit more complicated 
dynamics \cite{GeneralizedDynamicOneMax,EDAsTimo2016,gobel2018intuitive,9064720}.
A general polynomial tail bound is provided for these in \cite{kotzing_concentration_2016}, but any general exponential tail bounds for these processes are still missing.

\par In the following sections, we exploit the Optional 
Stopping Time Theorem to obtain our exponential tail 
bound. This theorem is crucial for proving the original 
additive drift theorem, as highlighted by~\cite {he_drift_2001}. This recurrence method can provide a 
different perspective to derive the exponential tail bound in 
runtime analysis. We defer the statements of Optional Stopping
Time Theorems in the appendix. 

\subsection{Previous Works and Discussion}
With the development of runtime analysis, researchers have 
established several concentration tail bounds for EAs. For 
example, \cite{lehre_general_2018} provides an exponential 
tail bound for the basic (1+1)-EAs on OneMax functions, which 
is a well-studied benchmark function to analyse the 
performance of EAs. 
To the best of our knowledge, the current 
best general tail bounds for both processes under the 
additive drift and variance-dominated processes can be found 
in \cite{kotzing_concentration_2016,lehre_general_2018}.

\cite{kotzing_concentration_2016} shows that the runtime is at most quadratic in $n$ with probability $1-p$ for any $p>0$. If we replace $1/ p^{\ell \log(c)}$ by $r>0$ and rewrite it in terms of an upper tail bound, then the original bound becomes that given two constants $ 1\leq c <n,\ell >0$ and for any $r>0$,
\begin{align}
    \Pr \left(T \geq r n^2  \right) \leq (\frac{1}{r})^{{1}/{\ell \log (c)}} \label{eq:timo1}. 
\end{align}

Although we have established a tail bound for variance-dominated processes that concentrate on the expectation in a polynomial order with respect to $r$ in Equation~(\ref{eq:timo1}), it is worth exploring whether a more precise concentration tail bound can be derived for such processes, such as an exponential tail. 
A sharper exponential tail bound can improve the expected runtime estimation and thus provide useful insights into randomised algorithms.

\section{A Recurrent Method in Upper Tail Bound}
Next, we explore how to derive a general framework for providing exponential tail bound for randomised algorithms, including evolutionary algorithms which satisfy certain conditions.  We explore the exit time of $X_{t}$ out of some interval $[0,b]$, using the same set-up as \cite{kotzing_concentration_2016}. 
\par In the proof of McDiamid inequality, ~\cite{mcdiarmid_1989} also uses the Hoeffding lemma and 
conditions on the past events to establish the recurrence. 
\cite{doerr_adaptive_2013} uses the multiplicative drift 
condition directly to build up the exponential recurrence 
relation and hence obtain an exponential tail bound. 
We want to borrow these ideas to derive an exponential tail bound for variance-dominated processes. To do this, we introduce the $k$-th hitting time of the target state, which is also used in the theorem (Theorem 
2.6.2) of \cite{menshikov_popov_wade_2016}.
We combine this recurrent method with the extended Optional Stopping Time theorem.

\subsection{Variance Overcomes Negative Drift w.h.p.}
\par We proceed to prove our main theorem by considering the most general variance drift theorem which overcomes some negative drift. Following the setting of \cite{hevia2023runtime}, in variance-transformed cases, we focus on the first hitting time of a discrete-time stochastic process $X_t$ at $0$ given that $X_t \in [0,b]$.
The proof of Theorem~\ref{thm:Uppertail2} uses Lemma~\ref{lem:stepsize}.

\begin{restatable}{lemma}{SecThreeLemOne}
\label{lem:stepsize}
Let $(X_{t})_{t\geq 0}$ be random variables over $\mathbb{R}_{\geq 0}$, each with finite expectation.  
Let $T$ be any stopping time of $X_{t}$.
If there exist constants $r,\eta>0$ with respect to $j,t$ such that for any $j\geq 0$, $\E{ \mathds{1}_{\{T > t\}} \mathds{1}_{\{|X_{t}-X_{t+1} |\geq j \}} \mid  \mathcal{F}_{t}} \leq r/(1+\eta)^j $, then there exists a positive constant $c$ such that $ \E{|X_{t+1}-X_t| \cdot  \mathds{1}_{\{T > t\}}\mid {\mathcal F}_t } \leq c$ for all $t \in \mathbb {N}\cup \{ 0 \}$.
\end{restatable}

\par By using Lemma~\ref{lem:stepsize}, we now satisfy condition (4) in the Optional Stopping Time Theorem, enabling us to proceed with the main proof (details in the appendix). 
Utilising the extended Optional Stopping Time Theorem, we follow a classic approach for generalisation, which also frees us from the fixed step size condition and the need for the Azuma-Hoeffding inequality for sub-Gaussian supermartingales.~\footnote{The proofs in \cite{kotzing_concentration_2016} mainly rely on Azuma-Hoeffding inequality for sub-Gaussian supermartingales.} 
We further construct a new stochastic process $Y_t=b^2-(b-X_t)^2+\delta t$, connected to the original process and the variance of the drift. 
By employing the idea of the $k$-th hitting time from \cite{menshikov_popov_wade_2016}, we obtain the upper bound for the $k$-th hitting time, allowing us to construct the recurrence.
We present the main theorem of this paper.

\begin{restatable}{theorem}{SecThreeMain}
\label{thm:Uppertail2}
Let $(X_t)_{t\in\N}$ be a sequence of random variables in a 
finite state space $S \subseteq \R$ adapted to a filtration $(\F_t)_{t\in\N}$, and let $T=\inf\{t\geq 0 \mid X_t\le 0\}$. 
Suppose
\begin{enumerate}[leftmargin=*,  align=left]
    \item[(A1)] there exist $\delta>0$ such that for all $t<T$, it holds that
    \begin{align*}
    \mathrm{E}_t\left((X_{t+1}-X_t)^2 -2(X_{t+1}-X_t)(b-X_t)
    \right)\ge \delta  \end{align*} 
    \item[(A2)] and for all $t\le T$, it holds that $0 \le X_t \le b$.
\end{enumerate}
Moreover, for all $t \geq 0$, assume there exist constants $r,\eta>0$ with respect to $j,t$,  for any $j\geq 0$, $\E{ \mathds{1}_{\{T > t\}} \mathds{1}_{\{|X_{t}-X_{t+1} |\geq j \}} \mid  \mathcal{F}_{t}} \leq {r}/{(1+\eta)^j}$.  Then, for $\tau >0$, $\Pr(T > \tau) \leq e^{-\tau \delta/e b^2}$.
\end{restatable}

Our main result (Theorem~\ref{thm:Uppertail2}) allows the increased tolerance of negative drift rather than non-negative drift tendency and only necessitates a constant second moment of drift, instead of variance, as outlined in \cite{kotzing_concentration_2016}. Consequently, we can establish a more precise exponential tail bound for stochastic processes with a constant second moment of the drift, even under weak, zero, or negative drift.

\subsection{Standard Variance Drift}
This section presents the standard variance drift scenario (Theorem~\ref{thm:ExpTail}) as a corollary of Theorem~\ref{thm:Uppertail2}. More precisely, we now restrict to the non-negative drift tendency. We first define several conditions which will be used later.

\begin{enumerate}[leftmargin=*, label=\arabic*), widest=7bis), align=left]
    \item[(C1*)] There exist constants $r,\eta>0$ with respect to $j,t$, such that for any $j\geq 0$ and for all $t \geq 0$,
                \begin{align*}
                    \E{ \mathds{1}_{\{T > t\}} \mathds{1}_{[|X_{t}-X_{t+1} |\geq j \}} \mid  \mathcal{F}_{t}} \leq \frac{r}{(1+\eta)^j}.
                \end{align*}
    \item[(C1)] There exists a constant $c>0$ such that $|X_{t}-X_{t+1}|<c $ for all $t \geq 0$.
    \item[(C2)] $\E{X_{t+1}-X_{t} \mid \mathcal{F}_{t}}\geq 0$ for all $t \geq 0$.
    \item[(C3)] There exists some constant $\delta>0$ such that $\E{X_{t+1}-X_{t})^2 \mid  \mathcal{F}_{t}} \geq \delta$ for all $t \geq 0$.
\end{enumerate}

\begin{restatable}{theorem}{SecThreeMainOneplus}
\label{thm:ExpTail}
Let $(X_{t})_{t\geq 0}$ be random variables over $\mathbb{R}_{\geq 0}$, each with finite expectation, such that conditions $(C1^*), (C2)$ and $(C3)$ hold. For any $b>0$, define $T= \inf \{t\geq 0 \mid X_{t}\geq b\}$.
If $X_0 \in [b]$, then $ \E{T} \leq  {(b^2-X_0^2)}/{\delta}$. Moreover, for $\tau>0$, $\Pr(T \geq \tau) \leq  e^{-\tau \delta/ eb^2  }$.
\end{restatable}

Theorem~\ref{thm:ExpTail} tells us that under the standard variance drift case as discussed in \cite{kotzing_concentration_2016}, we can derive an exponential tail bound for the runtime and such a process exhibits a high concentration around the expectation. 

\par Now, we present a corollary which consists of the fixed step size condition.
\begin{restatable}{corollary}{SecThreeMainCorOne}
\label{cor:ExpTail_cor} 
Let $(X_{t})_{t\geq 0}$ be random variables over $\mathbb{R}_{\geq 0}$, each with finite expectation which satisfy conditions $(C1), (C2)$ and $(C3)$. For any $b>0$, define $T= \inf \{t\geq 0 \mid X_{t}\geq b\}$.
Given that $X_0 \in [0,b]$, then $ \E{T} \leq  {(b^2-X_0^2)}/{\delta}$. 
Moreover, for $\tau>0$, $\Pr(T \geq \tau) \leq  e^{-\tau \delta / eb^2}$.
\end{restatable}

Furthermore, we derive a tail bound for the variance-dominated processes with two absorbing states. 
Following the setting of Theorem 10 in \cite{gobel2018intuitive}, we will prove the next theorem. 

\begin{restatable}{theorem}{SecThreeMainTwo}
 \label{thm:ExpTail_2absorbing}
 Let $(X_{t})_{t\geq 0}$ be random variables over $\mathbb{R}_{\geq 0}$, each with finite expectation such that $(C1*)$, $(C3)$ and $\E{X_{t+1}-X_{t} \mid \mathcal{F}_{t}}= 0$ hold.
For any $b>0$, define $T= \inf \{t\geq 0 \mid X_{t} \in \{0,b\} \}$. 
If $X_0 \in [0,b]$, then $ \E{T}  \leq {\left(X_0(b-X_0) \right)}/{\delta}$.
\par Moreover, for $\tau > 0$, $\Pr(T \geq \tau) \leq  e^{-2\tau \delta / eb^2 }$.
\end{restatable}

This proof is similar to Theorem~\ref{thm:Uppertail2} except that we use a different stochastic process $Y_t=X_t(b-X_t)+\delta t$ and show that $Y_t$ is a super-martingale. 
Unlike the proof of Theorem~\ref{thm:Uppertail2}, the proof of Theorem~\ref{thm:ExpTail_2absorbing} uses the extended Optional Stopping Time Theorem for super-martingale~\cite{williams1991probability}.
We defer the proof to the appendix.

\subsection{Standard Drift}

\par If a stochastic process has drift $\varepsilon$ where $\varepsilon>0$ is some positive constant, then we can give a different proof for 
the upper tail bound for additive drift from the proof in \cite{kotzing_concentration_2016}. 
This provides a more precise exponential upper tail bound.

\begin{restatable}{theorem}{SecThreeMainThree}
 \label{thm:ExpTail_drift}Let $(X_{t})_{t\geq 0}$ be random variables over $\mathbb{R}$, each with finite expectation which satisfy condition $(C1*)$ and $\mathbb{E}[X_{t+1}-X_{t} \mid  \mathcal{F}_{t}]\geq \epsilon$ for some $\epsilon > 0$.
For any $b>0$, define $T= \inf \{t\geq 0 \mid X_{t}\geq b\}$. If $X_0 \in [0,b]$, then $ E(T)  \leq  \frac{b-X_0}{\epsilon}$. Moreover, for $\tau >0$, $\Pr \left(T \geq \tau \right) \leq  e^{-\tau \epsilon/{eb } }$.
\end{restatable}

\par We recover the exponential upper tail bound for the additive drift theorem. The bound we obtain gets rid of the coefficients ${1}/{8c^2}$ in \cite{kotzing_concentration_2016}. Theorem 2.5.12 of \cite{menshikov_popov_wade_2016} provides a similar result and uses a similar recurrence proof idea. While our result generalises the result in \cite{menshikov_popov_wade_2016} by releasing the fixed step size, we provide a meaningful bound on the first hitting time instead of bounding it above by infinity.

\par In summary, we have discovered a simple alternative to the Azuma-Hoeffding inequality that provides an exponential tail bound by only relying on basic martingale theory. 
This result can be applied to the random local search (RLS) type algorithms that make finite steps at each iteration, as well as other randomised algorithms including evolutionary algorithms (EAs) that account for the possibility of large jumps occurring.
Another benefit from Theorem~\ref{thm:Uppertail2} is that it allows the tolerance of the negative drift up to $-\Omega(1/b)$.

\section{Applications to Random 2-SAT and Graph Colouring}

\par In this section, we illustrate our theorems on practical examples. 
We consider the examples provided by \cite{mcdiarmid_1993,mitzenmacher_probability_2005,gobel2018intuitive}.
which include variance-dominated processes. 
We first discuss the Random 2-SAT problem.

\subsection{Applications to Random 2-SAT}

\par The 2-SAT Algorithm is designed to solve instances of the 2-SAT problem, where a formula consists of clauses and each of them contains exactly two literals (either variables or negations).
In each iteration, the algorithm selects an unsatisfied clause and picks one of the literals uniformly at random. The truth value of the variable corresponding to this literal is then inverted. 
Repeat the process until either we meet the stopping criteria or a valid truth assignment is found. 

\par \cite{papadimitriou1991selecting} firstly provided a time complexity analysis on such a simple randomised algorithm that returns a satisfying
assignment of a satisfiable 2-SAT formula $\phi$ with $n$ variables. 
Later, \cite{gobel2018intuitive} recovered the results using drift analysis tools which we put in the Appendix.

By applying Theorem~\ref{thm:ExpTail} with a variance bound $1$, we can bound the number of function evaluations of order $O(n^4)$ with an upper exponential tail bound.

\begin{restatable}{theorem}{SecFourMainOne}
\label{thm:2sat2}
Given any $r\geq 0$, the randomised 2-SAT algorithm, when run on a satisfiable 2-SAT formula
over $n \in \mathbb{N}$ variables, terminates in at most $r n^4$ time
with probability at least $1-e^{-r/e}$.
\end{restatable}

\subsection{Applications to Graph Colouring}
 Now we consider graph colouring, which has already been studied by \cite{mcdiarmid_1993} and \cite{gobel2018intuitive}. 
The recolour algorithm generates a 2-colouring mapping with the condition that no monochromatic edges can be found.
The algorithm assumes a subroutine called SEEK, which, given a 2-colouring of the points, outputs a monochromatic edge if one exists. 
If there are no monochromatic edges, then the algorithm terminates.
Otherwise, the algorithm repeats picking a point uniformly at random from the given monochromatic edge and changes its colour.
\par \cite{gobel2018intuitive} provided a simpler proof of the $O(n^4)$ expected runtime of the recolouring algorithm for finding a 2-colouring with no monochromatic triangles on 3-colorable graphs. 
Following the setting and the proof of \cite{gobel2018intuitive}, by using Theorem~\ref{thm:ExpTail_2absorbing}, we can derive the following:

\begin{restatable}{theorem}{SecFourMainTwo}
\label{thm:recoloring}
Given any $r\geq 0$, the randomised Recolouring algorithm on a 3-colorable graph with $n \in \mathbb{N}$ vertices 
over $n \in \mathbb{N}$ variables, terminates in at most $r n^4$ time
with probability at least $1- e^{-4r/{3e}}$.
\end{restatable}

\section{Applications to Coevolutionary Algorithms}
Next, we consider a slightly complicated example: competitive co-evolutionary algorithms (CoEAs). 
Competitive co-evolutionary algorithms are designed to solve maximin optimisation or adversarial optimisation problems, including two-player zero-sum games~\cite{rozenberg_coevolutionary_2012,lehre_runtime_2022}. 
There are various applications, 
including CoEA-GAN \cite{toutouh_spatial_2019}, 
competitive co-evolutionary search heuristics on cyber security problem~\cite{DefendIT2023} and enhanced GANs by using a co-evolutionary approach for image translation~\cite{shu2019coeagan}.

We are interested in whether competitive CoEAs can help find Nash equilibrium efficiently. 
We use the formulation in \cite{nisan_roughgarden_tardos_vazirani_2007} to define Nash equilibrium. This paper focuses on Pure Strategy Nash Equilibrium (abbreviated NE).

\begin{defn}
(\cite{nisan_roughgarden_tardos_vazirani_2007})
Consider a two-player zero-sum game.
Given a search space $\X \times \Y$ and a payoff function $g: \X \times \Y \rightarrow \mathbb{R}$, if for all $(x,y) \in \X \times \Y$
\begin{align*}
    g(x,y^*) \leq g(x^*,y^*) \leq g(x^*,y).
\end{align*}
then $(x^*,y^*)$ is called a Pure Strategy Nash Equilibrium of a two-player zero-sum game.
\end{defn}

The pairwise dominance relation has been defined and introduced into a population-based CoEA in \cite{lehre_runtime_2022}.

\begin{defn}[Pairwise dominance]
\label{def:dominance-relation}
Given a function $g=\X \times \Y \rightarrow \R$ and two pairs $(x_1,y_1),(x_2,y_2)\in  \X \times \Y$, we say that $(x_1,y_1)$ dominates $(x_2,y_2)$ with respect to $g$, denoted $(x_1,y_1)\succeq_g(x_2,y_2)$, if and only if
$g(x_1,y_2)\ge g(x_1,y_1)\ge g(x_2,y_1)$.
\end{defn}

There is a single-pair CoEA called Randomised Local Search with Pairwise Dominance (RLS-PD)~\cite{hevia2023runtime}.
It has been shown that RLS-PD can find the NE of a simple pseudo-Boolean benchmark called \bilinear in expected polynomial runtime. 
The processes induced by RLS-PD on the \bilinear problem is exactly a variance-transformed process. 
We would like to use Theorem~\ref{thm:Uppertail2} to show the exponential tail bound for the runtime.

RLS-PD samples a search point (a pair of point $(x_1,y_1) \in \X \times \Y$) uniformly at random. 
In each iteration, RLS-PD uses the local mutation operator to generate the new search point where the local mutation operator is sampling a random Hamming neighbour. 
If the new search point dominates the original search point in a pairwise-dominance manner, then the original search point is replaced by the new one. Otherwise, the original search point remains the same.

\subsection{The \bilinear Problem}
In this section, we consider a simple class of discrete maximin benchmark called \bilinear, which was first proposed by~\cite{lehre_runtime_2022}.
In this work, we use a variation of \bilinear as \cite{hevia2023runtime} to simplify our calculation and illustrate applications of our main theorem. 
It has been empirically shown in \cite{hevia2023runtime} 
that RLS-PD behaves similarly on the original definition of \bilinear and the revised definition.
In this paper, we only consider this variation of \bilinear.

\begin{defn}
(\cite{hevia2023runtime})
The \bilinear function is defined for two parameters $\alpha, \beta \in (0,1)$ by
\begin{multline*}
    \bilinear_{\alpha,\beta}(x,y)
    :=\ones{y}(\ones{x}-\beta n) - \alpha n \ones{x} + E_1+E_2
\end{multline*}
with the error terms $E_1:= \max\{(\alpha n - \ones{y})^2,1\}/n^3$ and $E_2:=-\max\{(\beta n - \ones{x})^2,1\}/n^3$.
We also denote the set of Nash equilibria as $\OPT$, where
$\OPT:=\{(x,y)\mid \ones{x}=\beta n \wedge \ones{y}=\alpha n\}$.     
\end{defn}
We consider $\OPT$ as our solution concept and the problem setting $\alpha=1/2\pm O(1/\sqrt{n})$ and $\beta=1/2\pm O(1/\sqrt{n})$ as in \cite{hevia2023runtime}.
We now derive the exponential tail bound of RLS-PD to find the Nash equilibrium.

\subsection{RLS-PD solves \bilinear efficiently w.h.p.}
\begin{restatable}{theorem}{SecFiveMainOne}
\label{thm:RLS_runtime}
Consider $\alpha\in[1/2-A/\sqrt{n},1/2+A/\sqrt{n}]$ and $\beta\in [1/2-B/\sqrt{n},1/2+B/\sqrt{n}]$, where $A,B>0$ are constants and $3(A+B)^2\le 1/2-\delta'$ for some constant $\delta'>0$.
The expected runtime of  RLS-PD on $\bilinear_{\alpha,\beta}$ is $O(n^{1.5})$.  
Moreover, given any $r\geq 0$, the runtime is at most $2rn^{1.5}$, with probability at least $1- e^{-\Omega(r)}$.
\end{restatable}
We defer the proof of Theorem~\ref{thm:RLS_runtime} to the appendix. 

\par Theorem~\ref{thm:RLS_runtime} shows that RLS-PD can find the Nash Equilibrium in $O(n^{1.5})$ with overwhelmingly high probability. 
The exponential tail bound provides a stronger performance guarantee up to the tail of the runtime than the sole expectation.

\subsection{RLS-PD forgets the Nash Equilibrium w.h.p.}\label{sec:rls-pd-forgets}

After the algorithm finds a Nash Equilibrium efficiently, the inherent characteristics of the function cause the algorithm not only to forget the Nash Equilibrium but also move away from it by a distance $\Omega(\sqrt{n})$ in $O(n)$ iterations w.h.p. This is shown by the following theorem.

\begin{restatable}{theorem}{SecFiveMainTwo}
\label{thm:reaching_the_optimum_RLS}
Let $\alpha=1/2\pm A/\sqrt{n}$ and $\beta=1/2 \pm B/\sqrt{n}$, where $A,B>0$ are constants. 
Consider RLS-PD on $\bilinear_{\alpha,\beta}$. Then, for any initial search points $(x_0,y_0)$, the expected runtime that the search point firstly moves away from $\OPT$ by a Manhattan distance at least $(A+B) \sqrt{n}$  is $O(n)$.
Moreover, given any  $r>0$, the runtime is at most $rn$, with probability at least $1-e^{-\Omega(r)}$.
\end{restatable}
Theorem~\ref{thm:reaching_the_optimum_RLS} illustrates how drift analysis can expose weaknesses in algorithms, suggesting what needs to be improved in new algorithms.
In particular, we can see even though RLS-PD can find the optimum in polynomial time, it can still suffer from evolutionary forgetting (i.e. forget the optimum found in previous iterations) with high probability. 
So only the expected runtime estimate might be insufficient to determine whether a coevolutionary algorithm is good or not. 
This highlights the weakness of RLS-PD and the need to understand coevolutionary dynamics further. 

\section{Applications to Regret Analysis of a Bandit Learning Algorithm}
We start with a brief introduction of bandit problems. 
Suppose we have $K$ decisions or "arms", where we obtain the corresponding reward $r_a$ when we choose one specific decision $a$.
The goal of the bandit algorithm is to maximise cumulative reward among time horizon $T$ \cite{lattimore2020bandit,sutton2018reinforcement}. 
In this paper, we consider the quantity called regret (missed reward), which is the difference between the reward of the chosen arm and the optimal arm at each iteration. 
We provide a formal definition of regret as follows. 
\begin{defn}
(\cite{lattimore2020bandit,larcher2023simple})
\label{def:regret}
Given time horizon $T$, each arm $a$ is associated with a probability distribution $\mathcal{D}(a)$, which we assume to be over $[0,1]$ and for which the mean is denoted as $\mu(a)$; whenever arm $a$ is pulled, the agent receives a reward distributed according to $\mathcal{D}(a)$. The regret (missed reward) of the agent at round $t\in \mathbb{N}$ is defined as $R_t = r_{a^*} - r_{a_t}$, where $a_t$ is the arm chosen at round $t $, $r_a$ is the reward obtained from reward distribution $\mathcal{D}(a)$ and $a^*=\text{arg max}\{\mu(a)\mid a \in [K]\}$.
The goal of the agent is to minimise the total regret $\mathcal{R}=\sum_{t=1}^{T}R_t$ or the total expected regret $\E{\mathcal{R}}$. 
\end{defn}
This paper focuses on the non-stationary 2-armed bandit problem, in which the reward distributions may swap over time and $K=2$. More precisely, the agent receives a reward according to a reward distribution $\mathcal{D}(a_1)$ by pulling arm $a_1$ and another reward according to another reward distribution $\mathcal{D}(a_2)$ by pulling arm $a_2$. We assume that two distributions $\mathcal{D}(a_1)$ and $\mathcal{D}(a_2)$ are fixed, and they will swap if a change occurs along the time horizon. To simplify the calculation, we assume both distributions over $[0,1]$.
We present the application of our drift theorem (concentration tail bound) on regret analysis of a simple reinforcement learning algorithm for such a bandit problem. We defer the pseudo-code for Random Walk with Asymmetric Boundaries (\rwab) proposed by \cite{larcher2023simple} to the appendix. 
\par Note that \rwab is designed to balance the exploration and exploitation for non-stationary bandit problems. \rwab mainly relies on the \Challenge operator to determine which arm we prefer to pull or whether we swap the arms.
The \Challenge operator is designed to use the random walk of action value $S$ on $[-\sqrt{T/L},1]$. \cite{larcher2023simple} shows that the expected regret of \rwab is $\Theta \left(\sqrt{LT} \right)$ where $T$ is the time horizon and $L$ is the number of changes. 
 % As mentioned in Example~\ref{eg:GRprocess}, the expectation might be misleading in some case, and 
 We want stronger performance guarantees for the regret of \rwab, i.e. a concentration tail bound for the regret estimate. 
We would like to characterise the distribution of the regret.

Next, we present the main theorem for the regret of \rwab algorithm. In this theorem, we assume $L = o(T)$.

\begin{restatable}{theorem}{SecSixMain}
\label{thm:RegretMain}  Given any $\varepsilon \geq 1$,  the regret of Algorithm \rwab is at most $480 \varepsilon(L+\sqrt{LT})$ with probability at least $1- 2e^{-\sqrt{\varepsilon}/{e}}$.
\end{restatable}
The proof of Theorem~\ref{thm:RegretMain} is deferred to the Appendix. 

By Theorem~\ref{thm:RegretMain}, \rwab Algorithm has regret at most order $O \left(\sqrt{LT}\right)$ for a 2-armed non-stationary bandit problem w.h.p. By using minimax lower bound \cite{bubeck2012regret}, any algorithm on $K$-armed stationary bandit problems has regret at least ${\sqrt{Kn}}/{20}$ for time horizon $n$. In particular, for $K=2$, the lower bound for the expected regret of any algorithm on bandit stationary bandit problems is $\Omega (\sqrt{n})$ for time horizon $n$.
\cite{larcher2023simple} showed that by considering $L$ changes and each change of average steps $T/L$,  \rwab has expected regret at least $\Omega \left(L \cdot \sqrt{T/L} \right)= \Omega \left(\sqrt{LT} \right)$.
Theorem~\ref{thm:RegretMain} confirms that \rwab is optimal with overwhelmingly high probability, and we propose a new perspective of analysing a bandit algorithm by using drift analysis, which is rarely employed by the reinforcement learning community.

\section{Experiments}
To complement our asymptotic results with data for concrete problem sizes, we conduct the following experiments. 
\subsection{Empirical Evidence of RLS-PD on \bilinear}

We conduct experiments
with the RLS-PD for the maximin \bilinear problem.
The problem setup is $(\alpha,\beta)=(0.5,0.5),(0.3,0.3),(0.3,0.7)$, $(0.7,0.3),(0.7,0.7)$. These five scenarios cover the cases when the optimum lies in four different quadrants and the centre of the search space.
We set the mutation rate $\chi=1$ and problem size $n=1000$. We run $1000$ independent simulations for each configuration.
For each run, we initialise the search point uniformly at random.
\par  Figure~\ref{fig:totalRuntime0} displays the density plot for the runtime distribution of RLS-PD on \bilinear. 
The $x$-axis represents the runtime, the $y$-axis represents the frequency or density, and the red dotted line represents the average value of the runtime for each problem setting.
As $x$ increases, Figure~\ref{fig:totalRuntime0} 
shows that we have an exponentially decaying tail for the runtime of RLS-PD on each problem configuration. 
It is very unlikely that the runtime of RLS-PD on \bilinear deviates too much from the mean or the expected runtime from Figure~\ref{fig:totalRuntime0} for each problem configuration.  
From the statistics (figures and tables in the appendix), we can see for each configuration, the frequency that the actual runtime bounded above by the mean runtime converges to $1$.
When $(\alpha,\beta)=(0.5,0.5)$, the empirical results are consistent with our theoretical bounds in the sense of asymptotic order. 
The results for other problem configurations raise a conjecture about whether our theoretical results can also hold for all $\alpha,\beta \in [0,1]$.

\begin{figure}[!ht]
    \centering
  \subfloat[$\alpha=0.5$, $\beta=0.5$\label{1a0}]{%
       \includegraphics[width=0.48\linewidth]{{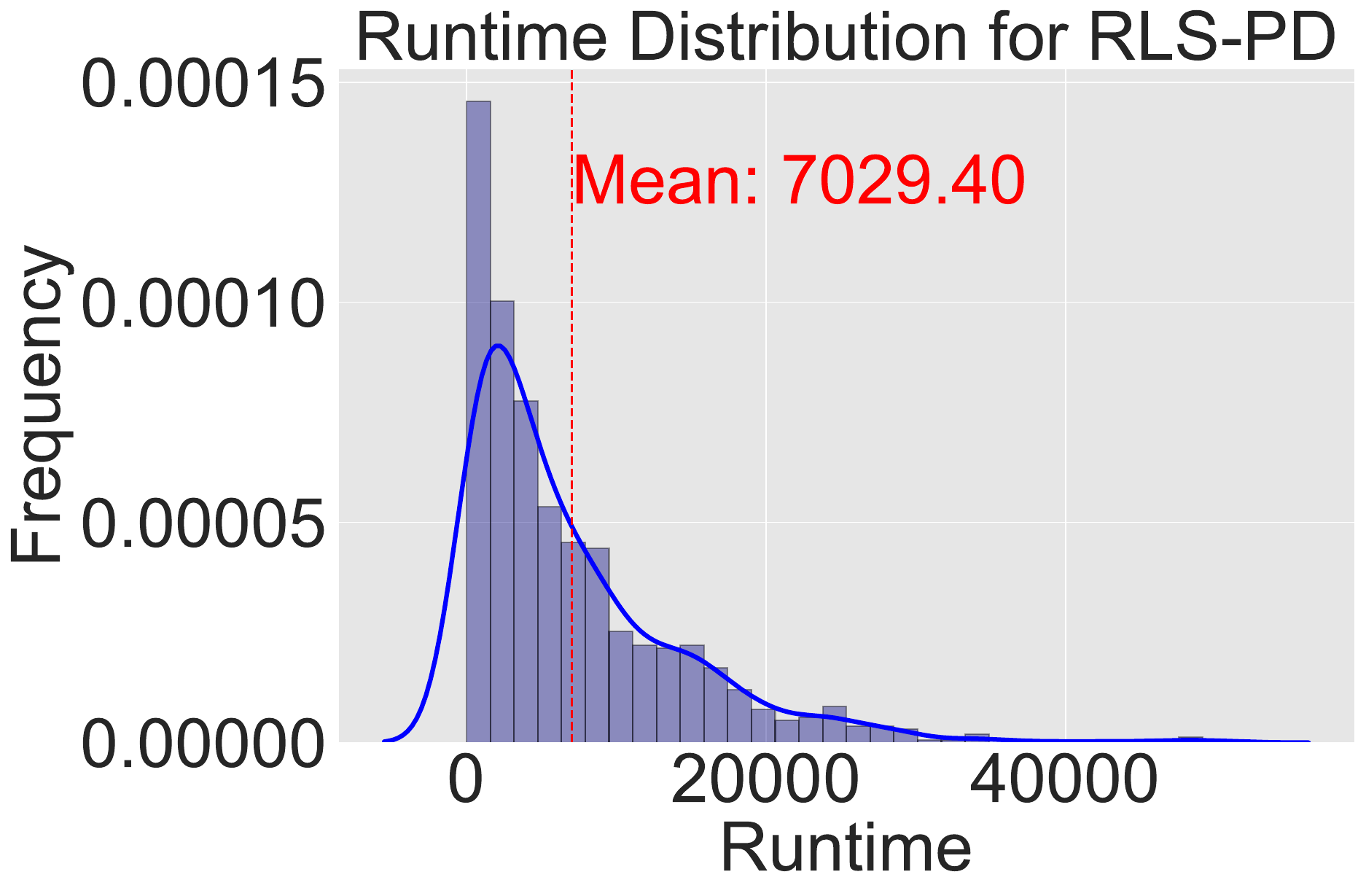}}}
    \hfill
  \subfloat[$\alpha=0.3$, $\beta=0.3$\label{1b0}]{%
        \includegraphics[width=0.49\linewidth]{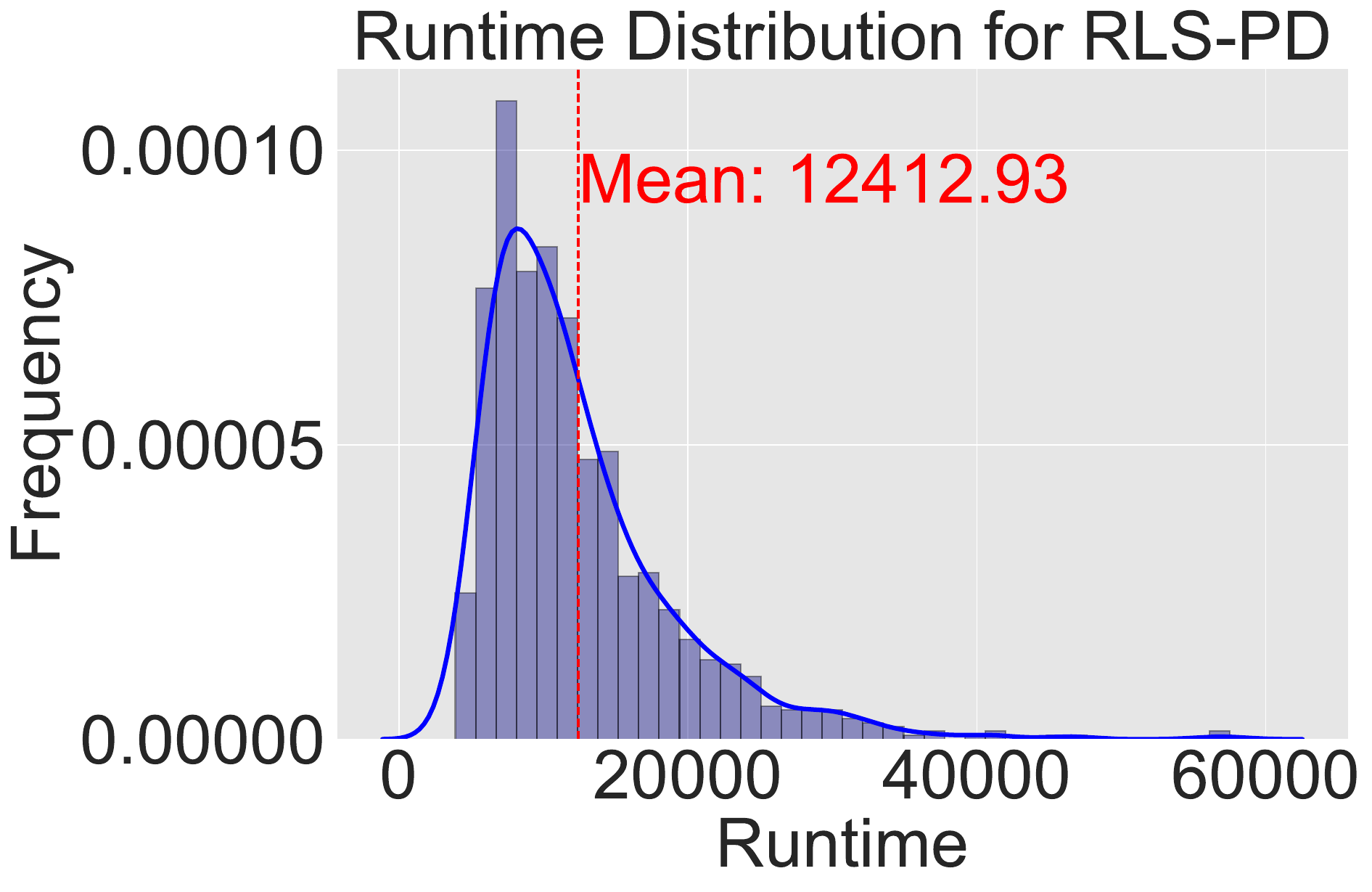}}
    \hfill
  \subfloat[$\alpha=0.7$, $\beta=0.7$\label{1c0}]{%
        \includegraphics[width=0.48\linewidth]{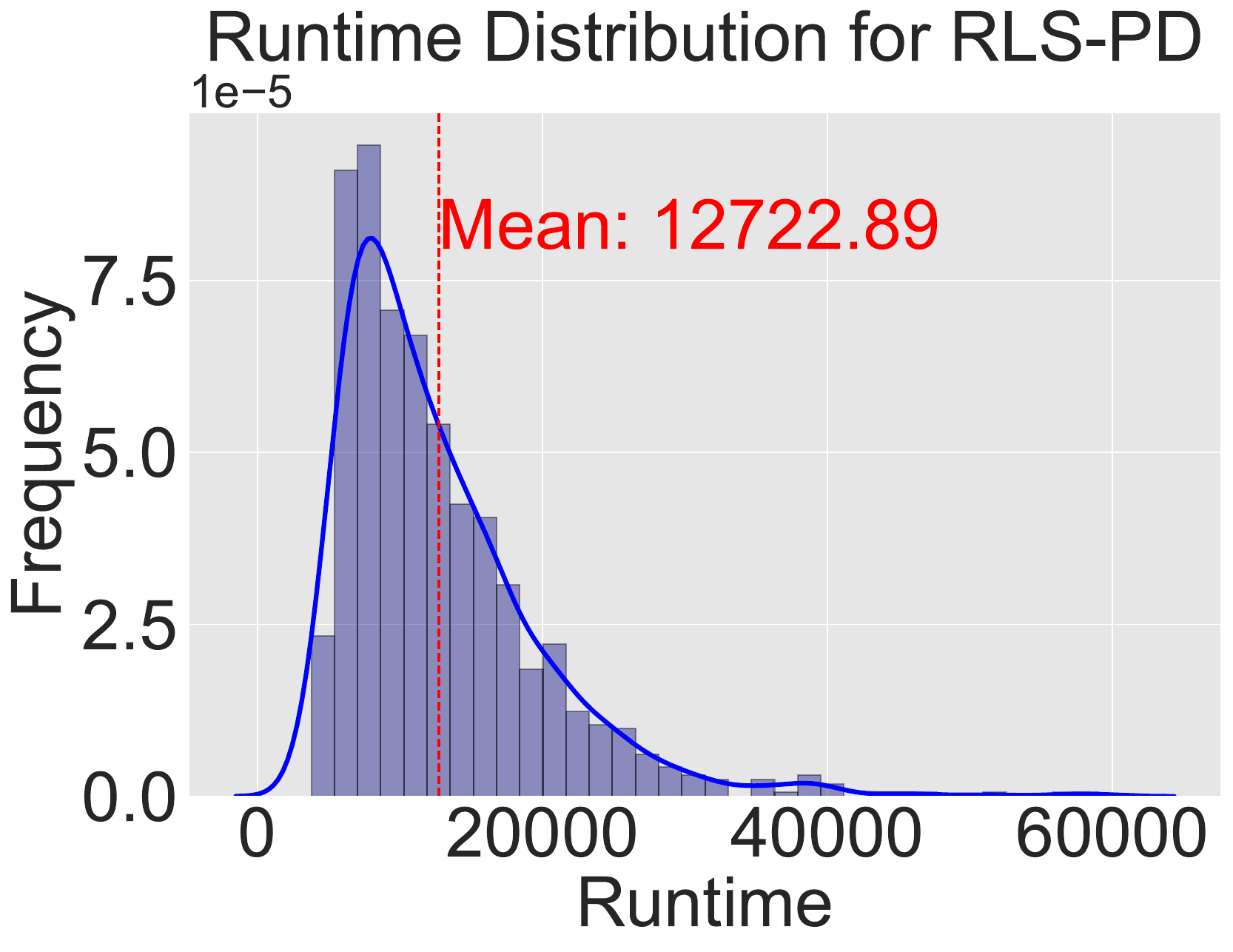}}
    \hfill
  \subfloat[$\alpha=0.3$, $\beta=0.7$\label{1d0}]{%
        \includegraphics[width=0.49\linewidth]{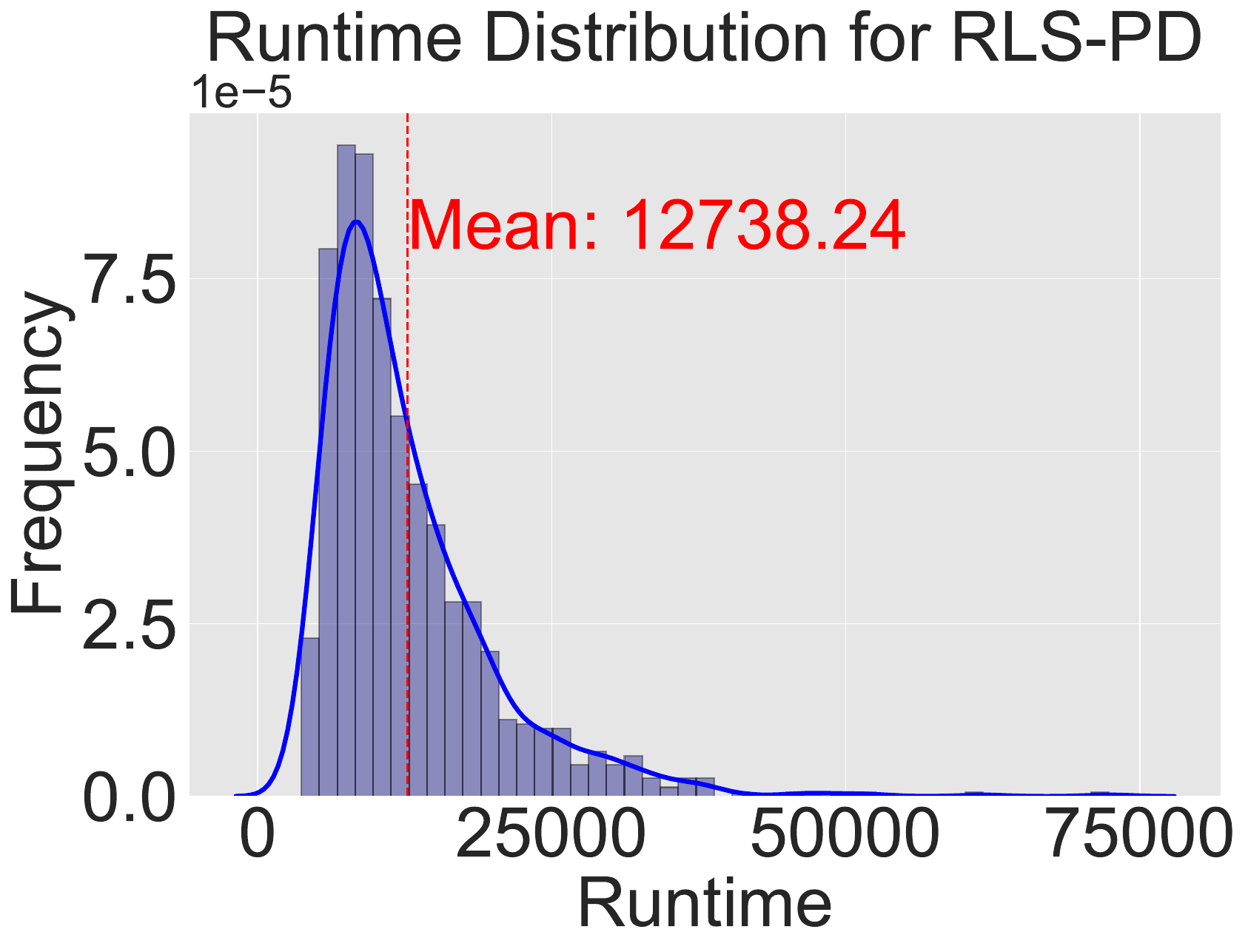}}
  \subfloat[$\alpha=0.7$, $\beta=0.3$\label{1e0}]{%
        \includegraphics[width=0.47 \linewidth]{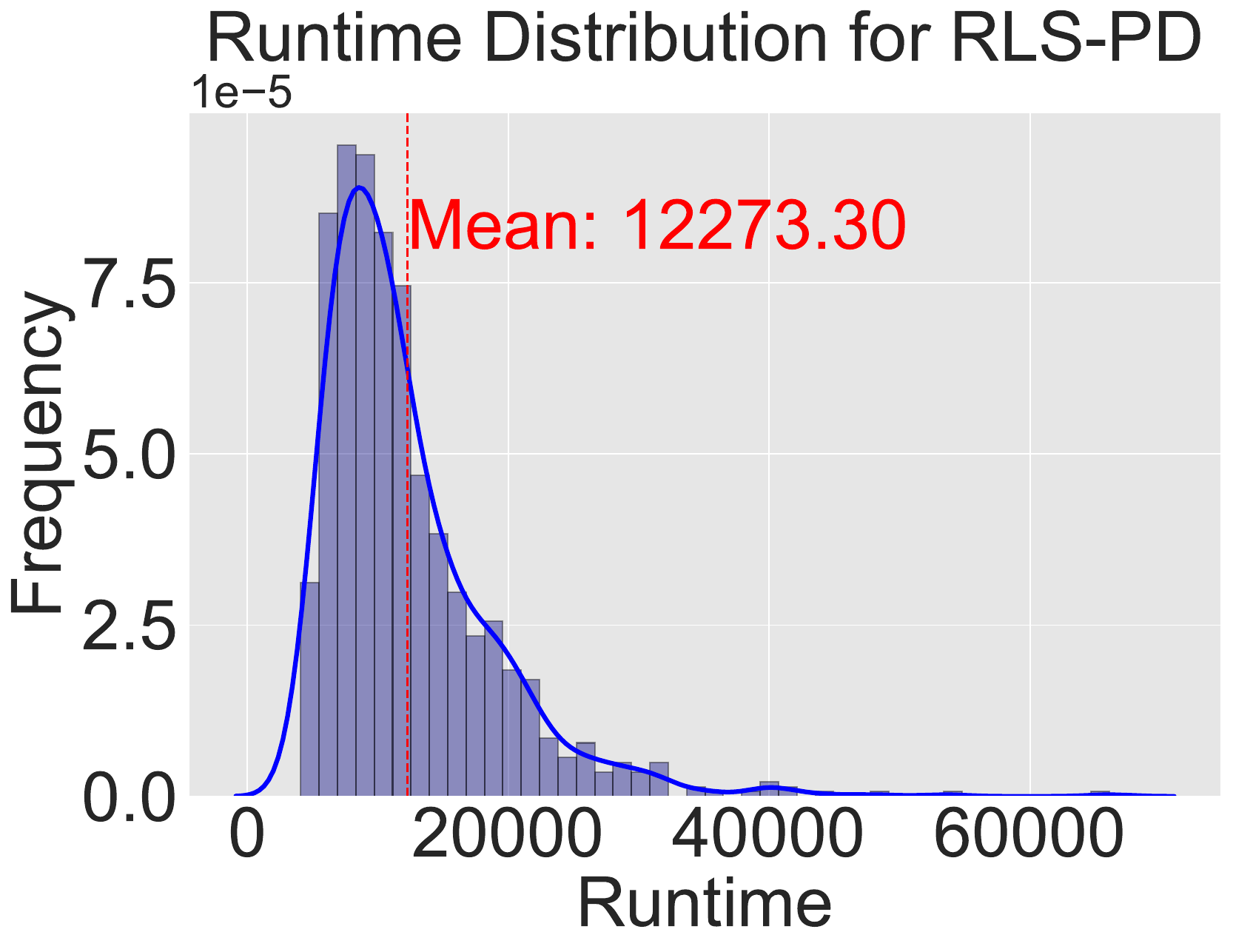}}
    \hfill
  \caption{Runtime distribution for RLS-PD for various $\alpha$ and $\beta$.}
  \label{fig:totalRuntime0} 
\end{figure}

\subsection{Empirical Evidence of \rwab Algorithm}
We conduct experiments with the \rwab Algorithm for the 2-armed non-stationary bandit problem. The environment is set up as two Bernoulli bandits with means $\mu_1=0.2, \mu_2=0.8$ and the number of changes $L=5, 10, 20, 40, 80, 100$. 
The changes are set up uniformly at random along the time horizon $T=1000$. $1000$ independent simulations are run for each configuration.

Figure~\ref{fig:totalRegret0} displays the regret distribution of \rwab. The $x$-axis represents the regret of \rwab, the $y$-axis represents the frequency or density, and the red dotted line represents the average regret for each problem setting. 
As $x$ increases, Figure~\ref{fig:totalRegret0} shows that we have an exponentially decaying tail for the regret of \rwab on 2-armed non-stationary bandit problem with respect to the changes $L$. Figure~\ref{fig:totalRegret0} shows the concentration of regret around the empirical mean or the expected regret and it is also unlikely that the regret of \rwab deviates too much from the expectation. 
The tables (deferred in the appendix) suggest that our theoretical tail bound is asymptotically tight regardless of the leading coefficient. 
Tables and figures for regret distributions show that for each configuration, the frequency that the actual regret bounded above by the mean regret is asymptotic to $1$ as the increases in the multiplicative factors of the upper bounds. 
Moreover, the convergence rate is significantly faster than the counterpart in the case of RLS-PD on \bilinear. 
This means that the theoretical bound (i.e. the leading coefficient) obtained has room to improve. One conjecture may be the process governing the dynamics of runtime for RLS-PD on \bilinear relies heavily on the high variance needed to overcome the negative drift, while the process governing the dynamics of regret induced by \rwab already exhibits positive drift everywhere before reaching the target state. 
Thus, it yields a faster convergence.

\begin{figure}[htpt] 
    \centering
  \subfloat[$T=1000, L=5$\label{2a0}]{%
       \includegraphics[width=0.48\linewidth]{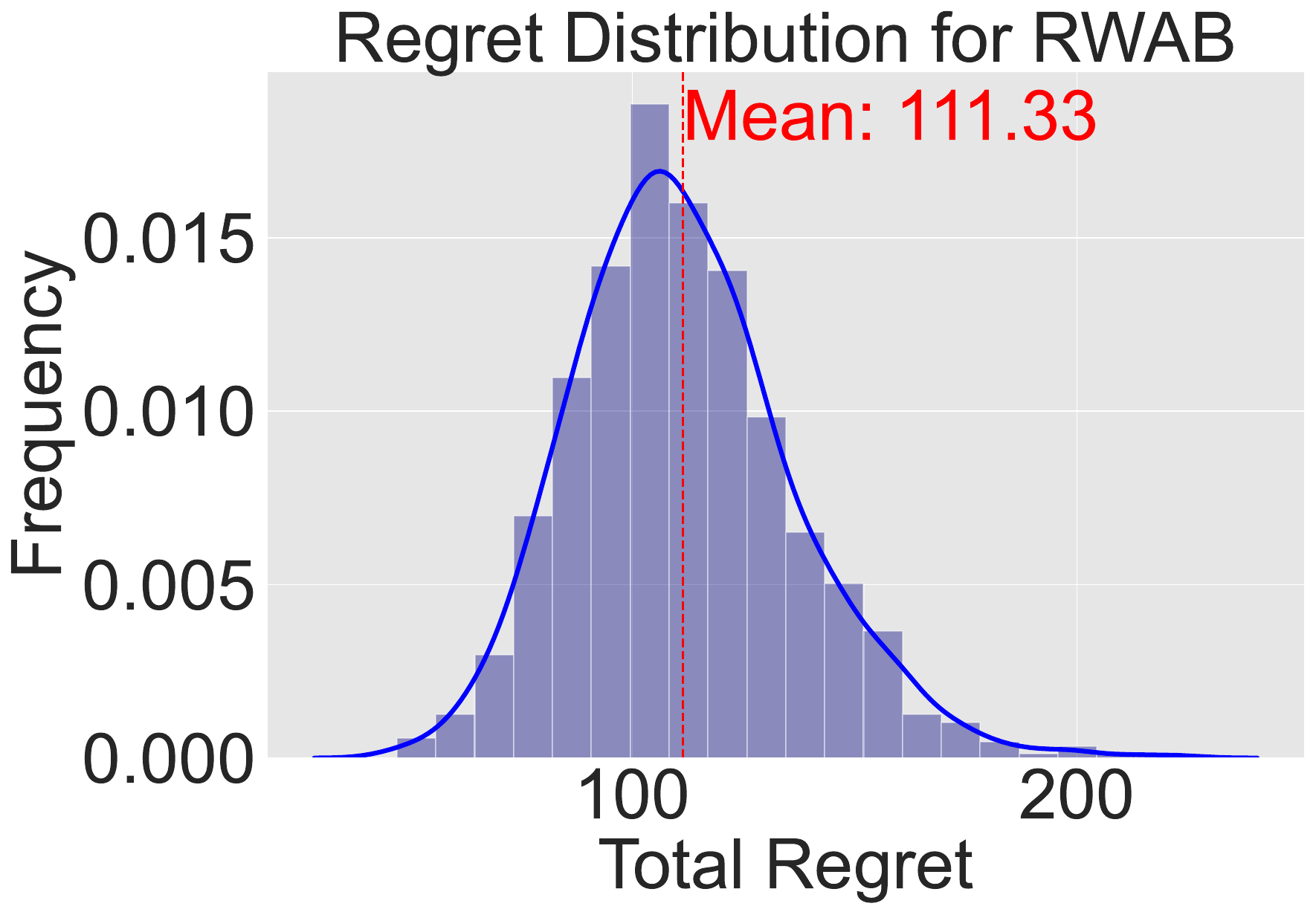}
       }
  \subfloat[$T=1000, L=10$\label{2b0}]{%
        \includegraphics[width=0.48\linewidth]{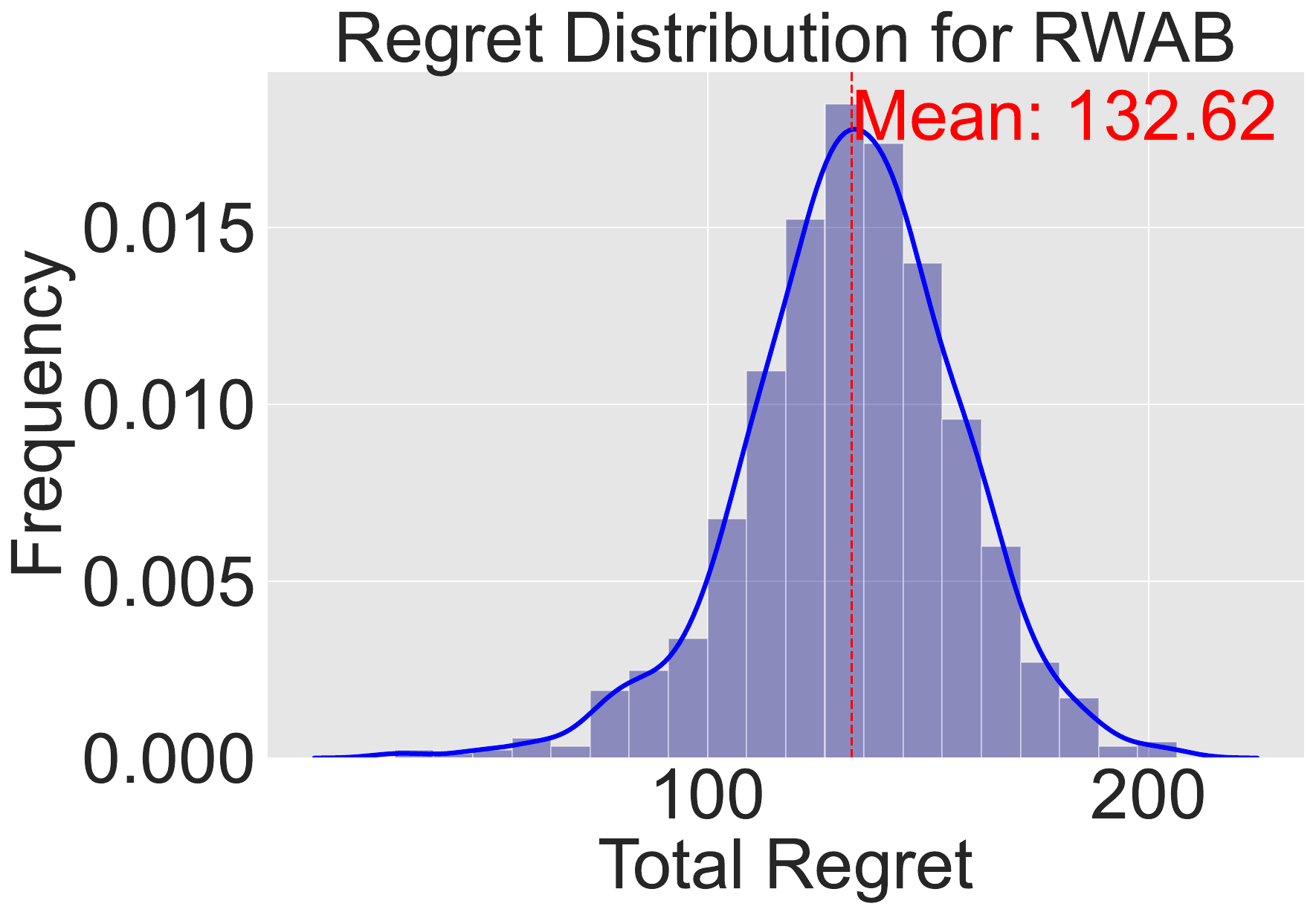}}
    \hfill
  \subfloat[$T=1000, L=20$\label{2c0}]{%
        \includegraphics[width=0.48\linewidth]{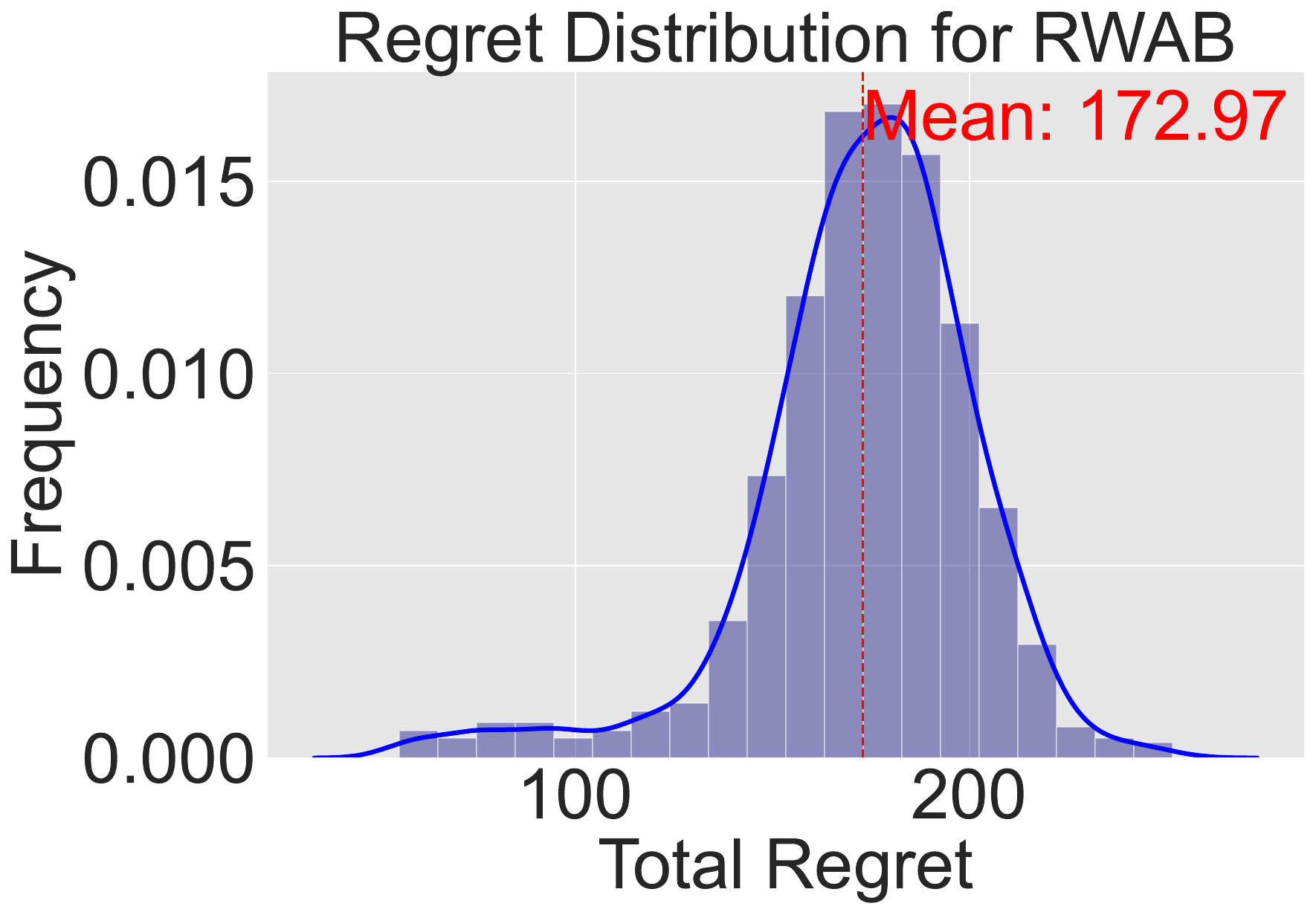}}
  \subfloat[$T=1000, L=40$\label{2d0}]{%
        \includegraphics[width=0.48\linewidth]{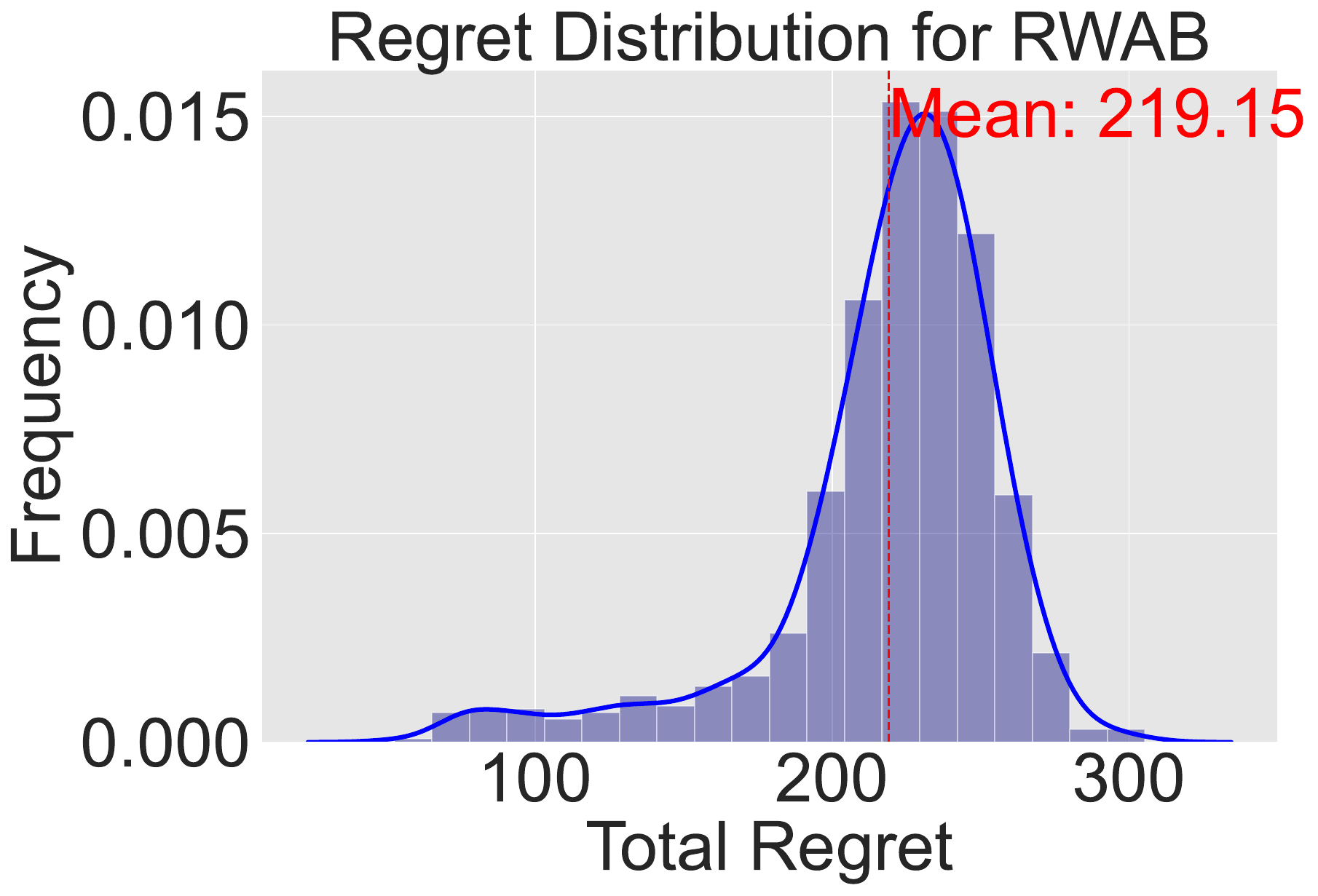}}
    \hfill
  \subfloat[$T=1000, L=80$\label{2e0}]{%
        \includegraphics[width=0.48\linewidth]{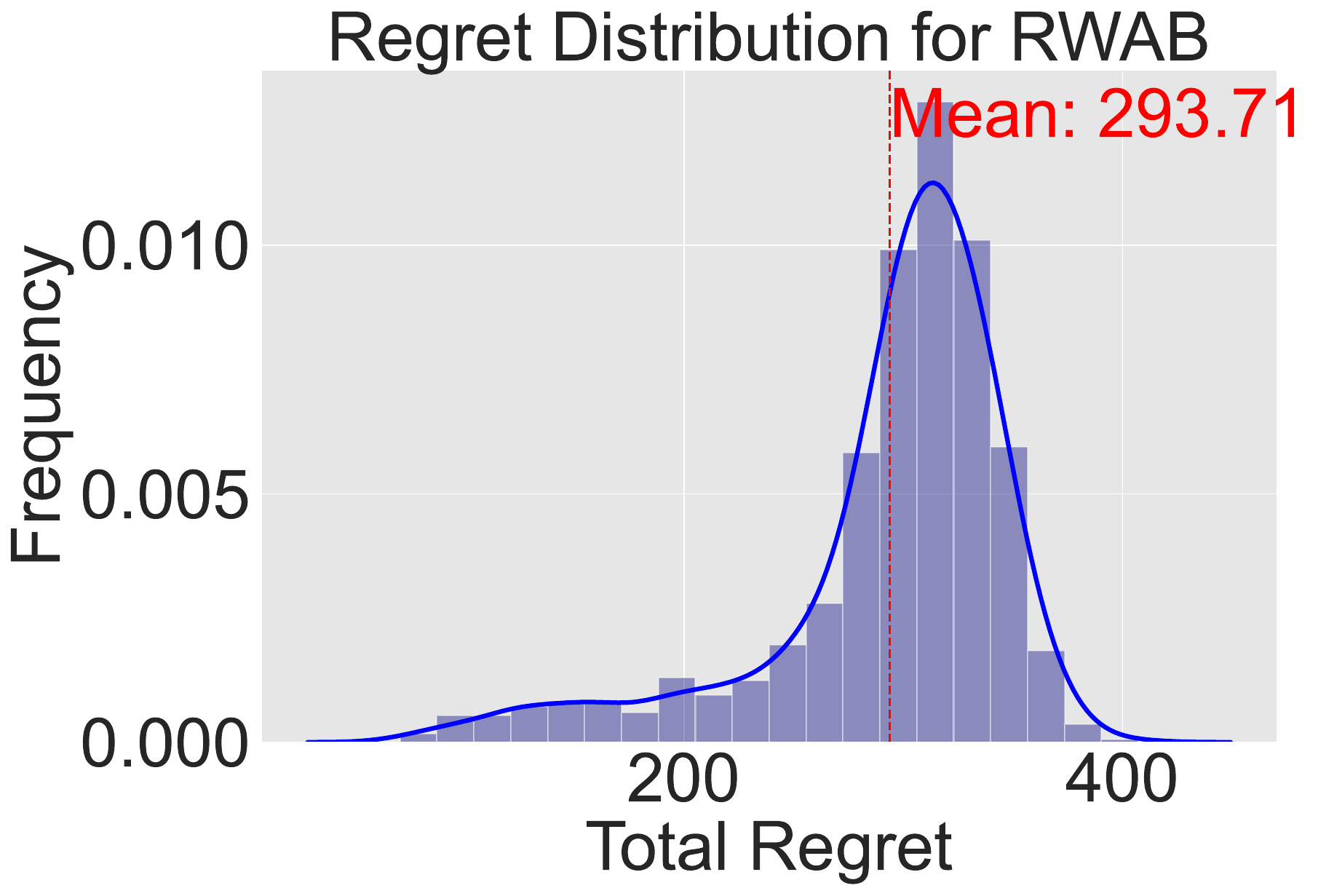}}
  \subfloat[$T=1000, L=100$\label{2f0}]{%
        \includegraphics[width=0.48\linewidth]{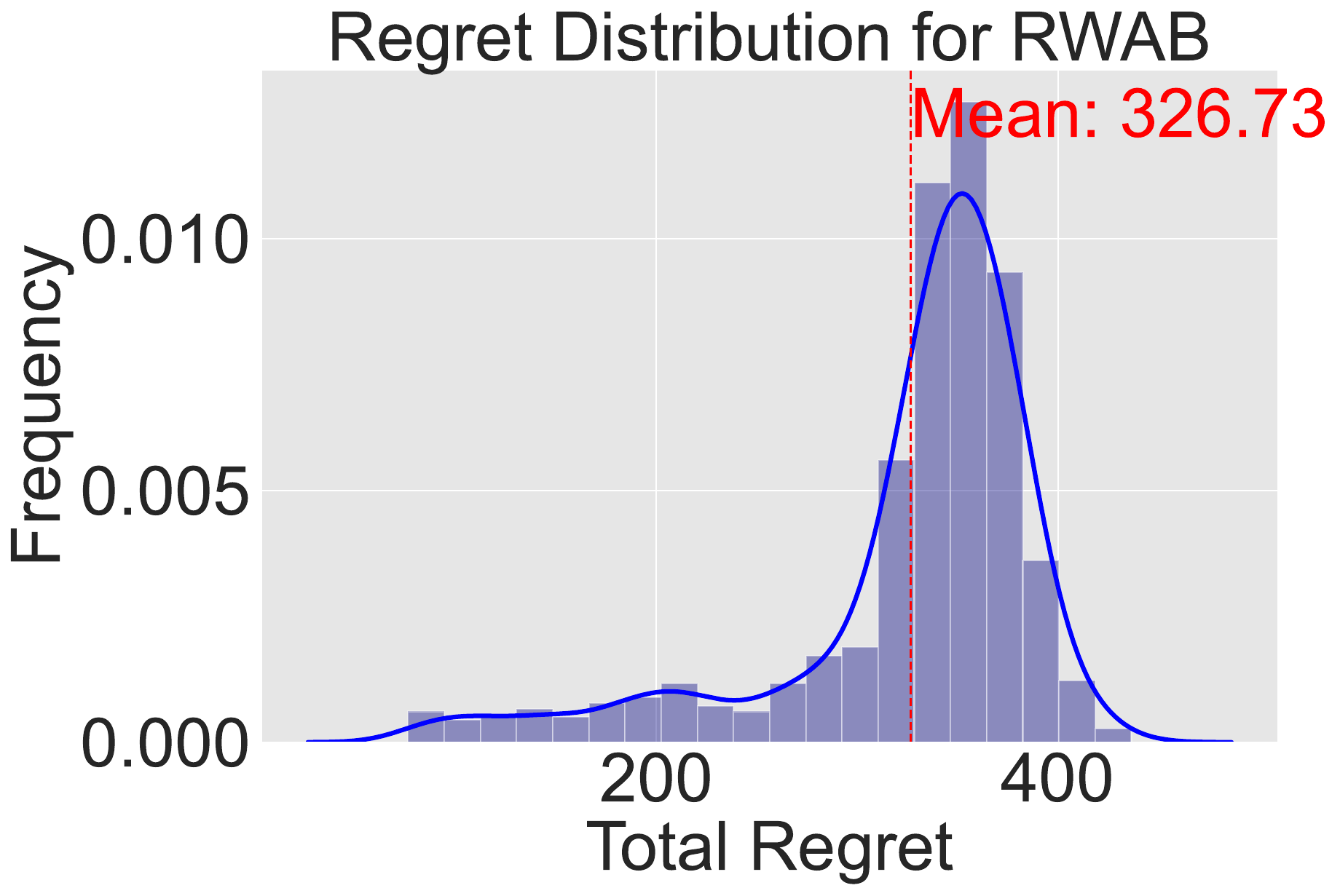}}
  \caption{Regret distribution for various values of $T$ and $L$.}
  \label{fig:totalRegret0} 
\end{figure}

\section{Conclusion}
This paper proves a more general and stronger drift theorem (tail-bound). 
Our theorems can be used to analyse the first hitting time of different random processes. 
As a sub-product, this paper also resolves the open problem left in \cite{kotzing_concentration_2016}, 
which asks for a suitable replacement for the Azuma-Hoeffding inequality to improve the tail bounds for random processes. 
We apply our theorems to several practical 
examples, including Random 2-SAT, Recolouring, 
competitive CoEAs and \rwab. To the best of our 
knowledge, it is the first tail-bound drift 
analysis of RLS-PD and \rwab.
Our drift theorems provide more precise information on how the runtime concentrates and a stronger performance guarantee. In practice, it shows the limitation of the current coevolutionary algorithm on maximin optimisation. It suggests a need for a deeper understanding of the mechanism of CoEAs, which may help to design a more stable CoEA. Moreover, our results confirm that randomness in \rwab can be helpful for stochastic non-stationary bandit problems. 
\par For future studies, both runtime analysis of CoEA on maximin optimisation and regret analysis of stochastic reinforcement learning algorithms via drift analysis are still poorly understood and unexplored areas.  
In particular, on the technical side, 
can we derive more precise bounds for \rwab since the leading coefficient seems not to be optimal from empirical results or
can we use these results to analyse more complicated CoEAs or bandit algorithms?
On the practical side, we could try to use such concentration bound to design more efficient algorithms. 
For example, we could try to design more stable CoEAs or develop a general optimal bandit algorithm by using the random-walk design analysed in this work.

\section*{Acknowledgments}
This work was supported by a Turing AI Fellowship (EPSRC grant ref EP/V025562/1).
The computations were performed using the University of Birmingham’s BlueBEAR high performance computing (HPC) service.

\bibliographystyle{named}

\newpage
\section{Pseudo-Code of Algorithms}
We consider the pseudocode from \cite{mitzenmacher_probability_2005}.
\begin{algorithm}[H]
  \caption{2-SAT Algorithm}
  \label{alg:2Sat}
  \begin{algorithmic}[1]
   \State Start with an arbitrary truth assignment.
   \For{$t=0,1,2,\ldots ,2mn^2$ until all clauses are satisfied}
    \State Choose an arbitrary clause that is not satisfied;
    \State Choose uniformly at random one of the literals in the clause and switch the value of its variable
    \EndFor
    \State  \textbf{if }{a valid truth assignment has been found} \textbf{then} return it.
    \State Otherwise, return that the formula is unsatisfied.
  \end{algorithmic}
\end{algorithm}

We consider the pseudocode from \cite{mcdiarmid_1993}.
\begin{algorithm}[H]
\caption{Recolour}
\label{alg:recolor}
\begin{algorithmic}[1]
\State Start with an arbitrary 2-colouring of the points
\While{SEEK returns a monochromatic edge $E$}
    \State pick a random point in $E$ and change its colour
\EndWhile
\end{algorithmic}
\end{algorithm}

We consider the pseudocode from \cite{hevia2023runtime}.
\begin{algorithm}[H]
\caption{RLS-PD: Randomised Local Search with Pairwise Dominance}
\label{alg:rls}
\begin{algorithmic}[1]
\Require Maximin-objective function $g:\X\times\Y\rightarrow\mathbb{R}$
\State Sample $x_1\sim\mathrm{Unif}(\{0,1\}^n)$
\State Sample $y_1\sim\mathrm{Unif}(\{0,1\}^n)$
\For{$t \in \{1,2,\dots\}$}
	\State Create $x',y' \in \{0, 1\}^n$ by copying~$x_t$ and~$y_t$ and flipping exactly one bit chosen uniformly at random from either $x_t$ or $y_t$.
    \If{$(x',y')\succeq_g(x_t,y_t)$}{ $(x_{t+1},y_{t+1}):=(x',y')$\label{line:create_new_pair}}
   \EndIf
\EndFor
\end{algorithmic}
\end{algorithm}

We consider the pseudocode from \cite{larcher2023simple}.
\begin{algorithm}[!ht]
\caption{Random Walk with Asymmetric Boundaries (RWAB)}
\label{alg:rwab}
\begin{algorithmic}[1]
\Require Time horizon $T$, number of changes $L$, two arms $a_1, a_2$ and Challenge probability $p=\sqrt{L/T}$.
\State Set $a^+ \leftarrow a_1$ and $a^- \leftarrow a_2$.
\For{$t \in \{1, 2, \dots, T\}$}
	\State Set $(a^+, a^-) \leftarrow \text{Challenge}(a^+, a^-)$ with prob. $p$.
    \State Otherwise pull $a^+$ once.
\EndFor
\end{algorithmic}
\end{algorithm}

\begin{algorithm}[H]
\caption{\Challenge}
\label{alg:rwab_challenge}
\begin{algorithmic}[1]
\Require Two arms $a^+, a^-$
\State Set $S \leftarrow 0$.
\While{$- \sqrt{T/L} < S < 1$}
	\State Pull both $a^+$ and $a^-$ once and observe rewards $r^+$ and $r^-$.
	\State Update $S \leftarrow S + r^+ - r^-$.
	\If{$S \geq 1$}
		\State \textbf{return} $(a^+, a^-)$ \text{ otherwise}
        \textbf{return} $(a^-, a^+)$
    \EndIf
\EndWhile
\end{algorithmic}
\end{algorithm}

\section{Optional Stopping Time Theorems}
We provide a formal definition of martingales/super-martingale / sub-martingale. 
\begin{defn}(Martingales~\cite{williams1991probability}) \label{def:martingale} A sequence of random variables $X_0,X_1,\dots$ is a martingale with respect to the sequence of filtration $\mathcal{F}_0, \mathcal{F}_1, \dots$ if for all $t\geq 0$, the following conditions hold:
\begin{enumerate}
    \item[(1)] $(X_t)_{t\geq 0}$ is adapted to $(\mathcal{F}_t)_{t\geq 0}$,
    \item[(2)] $E\left(|X_t| \right)<\infty, \forall t$
    \item[(3)] \label{eq:condition3} $E\left(X_t \mid \mathcal{F}_{t-1} \right)=X_{t-1}$ for $t\geq 1$.
\end{enumerate}
Moreover, a super-martingale with respect to $(\mathcal{F}_t)_{t\geq 0}$ is defined 
similar except that condition (3) is replaced by for $t\geq 1$,
$  E\left(X_t \mid \mathcal{F}_{t-1} \right)\leq X_{t-1}$.

A sub-martingale with respect to $(\mathcal{F}_t)_{t\geq 0}$ is defined with condition (3) replaced by for $t\geq 1$,
$E\left(X_t \mid \mathcal{F}_{t-1} \right)\geq X_{t-1} $.
\end{defn}

\begin{theorem}(Doob's Optional Stopping Time \cite{williams1991probability,doob_what_1971}) \label{thm:Doob} 
We consider the following conditions. Let $X_n$ be a stochastic process in $\mathbb{R}$ and $T$ be its stopping time.
\begin{enumerate}
    \item[(1)] $T$ is bounded (i.e. for some $N \in \mathbb{N}$, $T(\omega) \leq N, \forall \omega$);
    \item[(2)] $X$ is bounded (i.e. for some $K \in \mathbb{R}_{\geq 0}$, $|X_{t}(\omega) | \leq K$ for every $t$ and every $\omega$ and $T$ is a.s. finite );
    \item[(3)] $E(T)<\infty$, and, for some $K \in \mathbb{R}_{\geq 0}$,
        \begin{align*}
            |X_{t}(\omega)-X_{t-1}(\omega)| \leq K \quad \forall (t,\omega).
        \end{align*}
\end{enumerate}
In particular, let $X_{n}$ be a super-martingale and $\tau, \sigma$ be stopping times. 
Then $X_{\tau},X_{\sigma}$ are integrable and for $\sigma \leq \tau$, $E(X_{\tau}\mid \mathcal{F}_{\sigma}) \leq X_{\sigma}  \quad  a.s.$. In particular, for $\sigma=0, \tau=T$, $ E(X_{T}) \leq E(X_{0})$ in each of the above situations (1)-(3) holds.
\par Let $X_{n}$ is a sub-martingale. Then $X_{\tau},X_{\sigma}$ are integrable 
and for $\sigma \leq \tau$, $ E(X_{\tau}\mid \mathcal{F}_{\sigma}) \geq X_{\sigma} \quad a.s.$. In particular, for $\sigma=0, \tau=T$, $ E(X_{T})    
\geq E(X_{0})$ in each of the above situations (1)-(3) holds.
\end{theorem}

\par Notice that in condition (3), the theorem still assumes a fixed step size. Then, we need an extended version of the Optional Stopping Theorem~\cite{grimmett_probability_2001,bhattacharya2007basic}.

\begin{theorem}[Extended Doob's Optional Stopping Time]
\label{thm:Doob2}
We consider the following conditions. Let $X_n$ be a stochastic process in $\mathbb{R}$ and $T$ be its stopping time.
\begin{enumerate}
    \item[(1)] $T$ is bounded (i.e. for some $N \in \mathbb{N}$, $T(\omega) \leq N, \forall \omega$);
    \item[(2)] $X$ is bounded (i.e. for some $K \in \mathbb{R}_{\geq 0}$, $|X_{t}(\omega) | \leq K$ for every $t$ and every $\omega$ and $T$ is a.s. finite );
    \item[(3)] $E(T)<\infty$, and, for some $K \in \mathbb{R}_{\geq 0}$,
        \begin{align*}
            |X_{t}(\omega)-X_{t-1}(\omega)| \leq K \quad \forall (t,\omega).
        \end{align*}
    \item[(4)] The stopping time $T$ has a finite expectation and the conditional expectations of the absolute value of the martingale increments are almost surely bounded. More precisely, $E(T)<\infty$, and there exists a constant $c>0$ such that 
        \begin{align*}
           E\bigl[|X_{t+1}-X_t|\,\big\vert\,{\mathcal F}_t\bigr]\le c
        \end{align*}
    almost surely on the event $\{T > t\}$ for all $t \in \mathbb {N}_0$.
\end{enumerate}
Theorem~\ref{thm:Doob} holds in each of the above situations (1)-(4) holds. 
\end{theorem}

\section{Definitions and Theorems}
\subsection{Supplementary material for the analysis of Random 2-SAT and Recolouring algorithms}
We defer some important results of Random 2-SAT and Recolouring algorithms here.

\begin{theorem}(Random 2-SAT \cite{papadimitriou1991selecting}) \label{thm:2sat}
The randomised 2-SAT algorithm, when run on a satisfiable 2-SAT formula
over $n \in \mathbb{N}$ variables, terminates in $O(n^4)$ time in expectation.
\end{theorem}

\begin{theorem}(McDiarmid \cite{mcdiarmid_1993}) 
\label{thm:McDiarmid} The expected run time of Recolouring on a 3-colorable graph with $n \in \mathbb{N}$ vertices is $O(n^4)$.
\end{theorem}

\subsection{Supplementary material for the analysis of coevolutionary algorithm}
We defer some important definitions and lemmas used in the analysis of RLS-PD algorithm here. Following the definitions used in \cite{hevia2023runtime}, we define

\begin{defn}
\label{def:forward_backward_drift}
Let $M(x,y)$ be the Manhattan distance from a search point $(x,y)\in \mathcal{X}\times \mathcal{Y}$ to the optimum, that is, $M(x,y):=\abs{n\beta -\ones{x} }  +  \abs{n\alpha- \ones{y} }$. Let $M_t:=M(x_t,y_t)$. For all $t \in \mathbb{N}$ we define:
  \begin{align*}
      \pplus&:=\prob{M_{t+1}>M_t\mid M_t=M(x,y)};\\
      \pminus&:=\prob{M_{t+1}<M_t\mid M_t=M(x,y)}.
  \end{align*}
\end{defn}

\begin{defn}
\label{def:quadrant}
During our analysis, we divide the search space into four quadrants. We say that a pair of search points $(x,y)$ is in:
\begin{itemize}
    \item the first quadrant if $0\le \ones{x}<\beta n \wedge \alpha n \le \ones{y}\le n$, 
    \item the second quadrant if $\beta n \le \ones{x}\leq n \wedge \alpha n <\ones{y}\leq n$,
    \item the third quadrant if $\beta n <\ones{x}\leq n \wedge 0 \leq \ones{y} \le \alpha n$, and
    \item the fourth quadrant if $0\leq \ones{x} < \beta n \wedge 0\leq \ones{y} <\alpha n$.
\end{itemize}

\end{defn}

\begin{lemma}(\cite{hevia2023runtime})
\label{lem:drift_Manhattan_RLS}
Consider RLS-PD on \bilinear as in Theorem~\ref{thm:RLS_runtime}. Define $T:=\inf\{t\mid\M_t = 0\}$, then for every generation $t<T$
\begin{align*}
    \E{\M_{t}-\M_{t+1}- \frac{\M_{t}-(A+B)\sqrt{n}}{2n};t<T \mid M_t} \ge 0.
\end{align*}
\end{lemma}

\subsection{Supplementary material for the analysis of \rwab algorithm}
We defer some important definitions used in the analysis of the \rwab algorithm here. Following the definitions used in \cite{larcher2023simple}, we define

\begin{defn}
\label{def:swap}
We say that $a^+$ swaps or we have a swap if in the call of \Challenge, the value of $a^+$ changes. A swap is called a mistake when $a^+=a^*$ at the start of \Challenge and no change of the underlying distribution of the rewards (abbrev. change) occurs while the \Challenge was running.
\end{defn}

\begin{defn}
\label{def:era}
An era denotes the time between two consecutive changes in the reward distribution. A sub-era is an interval in which both the underlying reward distribution and the arm $a^+$ remain constant. We use $L$ to denote the number of eras and $M$ denote the number of sub-eras.
\end{defn}

Before showing the main theorem, we need some lemma to proceed.

\begin{restatable}{lemma}{SecSixLemOne}
\label{lem:RegretLemOne} Given $X\sim Geo(p), Y\sim Geo(q)$ where $p,q \in (0,1)$, if $p \geq q$, then $Y$ stochastic dominates $X$.

\end{restatable}

\par Lemma~\ref{lem:RegretLemOne} shows that given two Geometric distributed random variables, the sufficient condition for one stochastic dominating the other is simply with higher "success" probability. 

The proof consists of four cases. We denote the accumulated regret for each case by $\mathcal{R}_i$ for $i \in [4]$. Next, following the definitions in \cite{larcher2023simple}, we define $\mathcal{R}_i$ for $i \in [4]$ formally.
\begin{defn}
We denote the optimal arm by $a^*$, the action value of $S$ in \Challenge by $S_t$ and the expected regret at iteration $t$ by $\Delta$. Let $s=\sqrt{T/L}$ where $T$ is the time horizon and $L$ is the number of changes. Then, we define
\begin{itemize}
    \item $\mathcal{R}_1$:  ``The accumulated regret when starting \Challenge with correct arm $a^{+}=a^*$ until a change occurs or the \Challenge ends'';
    \item $\mathcal{R}_2$: ``The accumulated regret when
    a change occurs during an ongoing \Challenge until the next change occurs or the end of this ongoing \Challenge no matter if we have $a^+=a^*$ or $a^+\neq a^*$'';
    \item $\mathcal{R}_3$: ``The accumulated regret when $a^+ \neq a^*$ with no ongoing \Challenge'';
    \item $\mathcal{R}_4$: ``The accumulated regret when starting \Challenge with incorrect arm $a^{+}\neq a^*$ until a new change occurs or the end of \Challenge in which the action value in \Challenge, denoted by $S_t$, hits $-s$ to trigger a swap''.
\end{itemize}
\end{defn}

\begin{proof}[Proof of Sketch for Theorem~\ref{thm:RegretMain}]
Here we provide a sketch of the proof.   
We divide the analysis into four cases covered by $\mathcal{R}_i$ for $i \in [4]$ (as shown in Figure~\ref{fig:RegretClass}).
Once we obtain the tail bound for these $\mathcal{R}_i$, we can derive the tail bound for total regret $\mathcal{R}$ by using $\mathcal{R} \leq \sum_{i=1}^4 \mathcal{R}_i$.
\end{proof}

\begin{figure}[!ht] 
    \centering
       \includegraphics[width=1.0\linewidth]{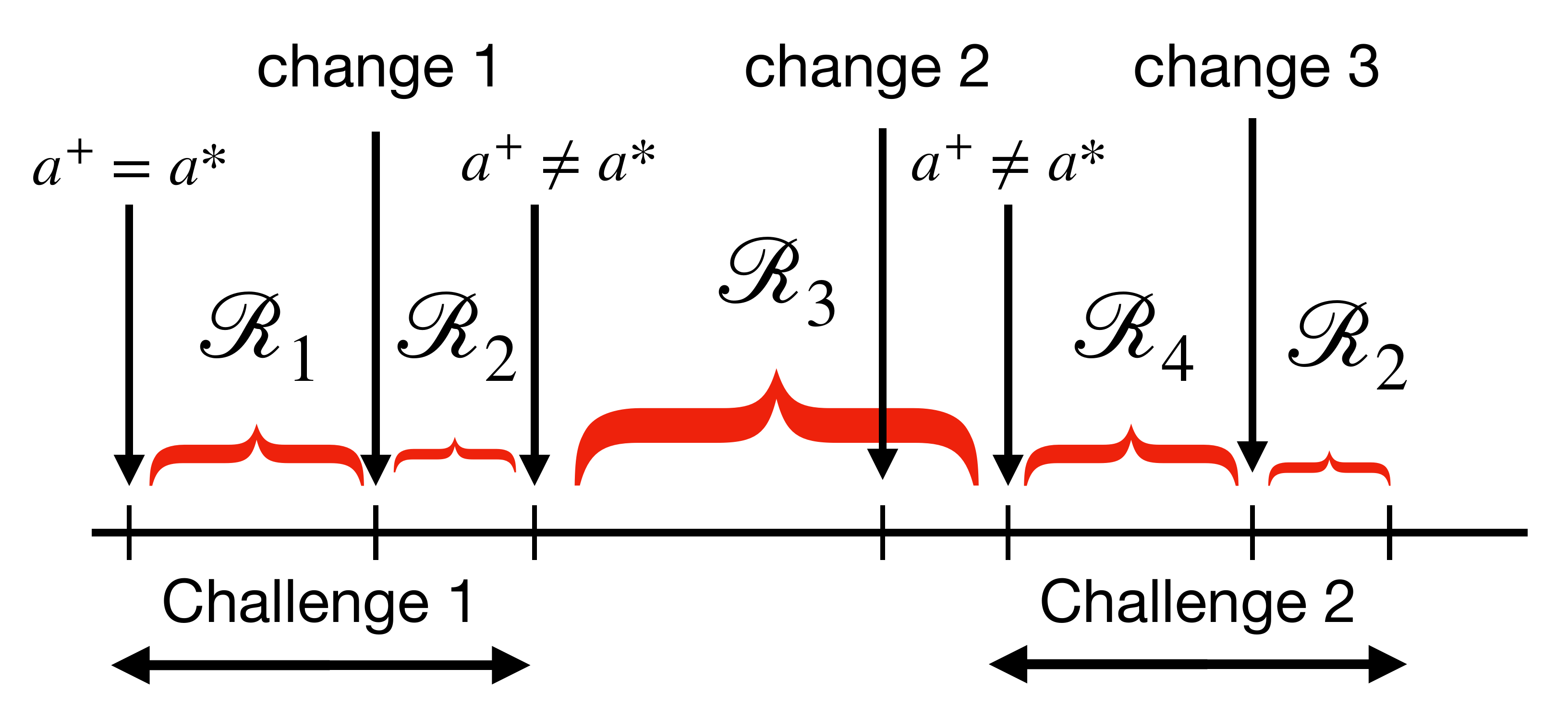}
  \caption{Classes of accumulated Regret along time horizon.}
  \label{fig:RegretClass} 
\end{figure}

\newpage
\section{Omitted Proofs}

\SecThreeMain*
\begin{proof}[Proof of Theorem~\ref{thm:Uppertail2}] For any $ k\in \mathbb{N}$, we define $T_{k} = \inf \{t \geq k \mid X_{t} \leq 0 \}$. Note that for any $k$ and $m$, if $m<k$, then  $\{\omega\in\Omega \mid T_k(\omega)=m\}=\emptyset\in \mathcal{F}_m$ and if $m \geq k$, from definition of $T_k$ (consider the process after $X_k$), then $\{\omega\in\Omega\mid T_k(\omega)=m\}=\{\omega\in\Omega\mid X_i(\omega)>0 \text{ and } X_m(\omega)\leq 0 \text{ for $i=k, \dots, m-1$} \}$ is in the $\sigma$-algebra $\mathcal{F}_m$ since $X$ is adapted to $\mathcal{F}$.
Thus, $T_k$ is a stopping time with respect to $\left(\mathcal{F}_t\right)_{t\geq 0}$. Since $X_t$ has a negative drift, we want to transform the process to overcome such negative drift by using variance. We proceed by defining $Y_t$. Also, we define $Z_t$ to connect $Y_t$ and $T_k$ to establish the recurrence step.
\par We would like to apply the extended optional stopping theorem for the random variable $X_t, Z_t$. To do so, we first need to prove that the expectation of $T$ is finite and thus the expectation of $T_k$ is finite for any given $k$. We want to apply additive drift theorem~\cite{he_drift_2001} on $Y_t$ instead of directly on $X_t$. 
Note that if $X_t \leq 0$ and $X_t \in [0,b]$, then $Y_t=b^2-(b-X_t)^2 \leq 0$. If $Y_t \leq 0$, then $b^2-(b-X_t)^2\leq 0$ which implies that $b^2 \leq (b-X_t)^2$. $X_t \in [0,b]$ and $b-X_t \in [0,b]$, then we have $b \leq b-X_t$ and thus we get $X_t \leq 0$. We can see the equivalence of these two events from here. So $Y_t \leq 0$ is equivalent to $X_t \leq 0$ from the definition of $Y_t$ and the equivalence of these two events gives
\begin{align*}
    T=\inf\{t \geq 0 \mid X_t   \leq 0\} = \inf\{t \geq 0 \mid Y_t  \leq 0\}.
\end{align*}
Consider the process $Y_t=b^2-(b-X_t)^2$ for all ${t\in\N}$, we have
\begin{align}
    \Et{Y_t-Y_{t+1}} &= \Et{\left(b-X_{t+1} \right)^2-\left(b-X_t \right)^2} \nonumber \\
                     &=\Et{X_{t+1}^2-X_t^2-2b\left(X_{t+1}-X_t \right)} \nonumber\\
                     &=\mathrm{E}_t \left( \left(X_{t+1}-X_t \right)^2 \right. \nonumber\\
                     &\quad \left. -2\left(X_{t+1}-X_t \right)\left(b-X_t \right) . \right)\nonumber\\ \intertext{Using $(A_1)$ gives}
                     &\geq \delta >0.  \label{eq:PositiveDriftYt}
\end{align}
 By classic additive drift theorem~\cite{he_drift_2001} on $Y_t$ we can derive the upper bound for the expected runtime is $\E{T} \leq \frac{b^2-(b-X_0)^2}{\delta}$. 
 Now, we define another process $Z_t= Y_t+\delta t$, then
\begin{align*}
    \Et{Z_t-Z_{t+1}}&= \Et{Y_t-\delta t-Y_{t+1}-\delta(t+1)} \\                 &=\Et{Y_t-Y_{t+1}-\delta} \\ \intertext{Using Eq~(\ref{eq:PositiveDriftYt}) gives}
                        &\geq 0.
\end{align*}
Also notice that $Y_t, Z_t$ are adapted to the filtration $\mathcal{F}_t$ and due to $X_t$ is finite then $\E{|X_t|}<\infty$. This implies that $\E{|Y_t|}, \E{|Z_t|}$ are bounded as well for any given $t$. Then, both $Z_t$ and $Y_t$ are super-martingales according to Definition~\ref{def:martingale}. 
Since we have shown that $T_k$ is a stopping time, we can apply Theorem~\ref{thm:Doob2} with respect to $Z_t$.
\begin{align*}
    \Ek{Z_k-Z_{T_k}} &\geq 0 \\ \intertext{Substituting $Z_t$ gives}
    \Ek{Y_k+\delta k -\left(Y_{T_k}+\delta T_k \right)} &\geq 0 \\ \intertext{Rearranging the equation gives}
    \Ek{Y_k -Y_{T_k} } &\geq \delta \Ek{T_k-k }
\end{align*}

Now, note that 
\begin{align*}
    0\leq \Ek{Y_{T_k}} &= \Ek{b^2-(b-X_{T_k})^2} \\
                       &= b^2 - \Ek{(b-X_{T_k})^2} \\ \intertext{Using Jensen's inequality gives,}
                                                &\leq b^2 - (\Ek{b-X_{T_k}} )^2 \\ \intertext{Since $\E{T}<\infty$ as shown and we have shown Lemma~\ref{lem:stepsize}, we satisfy condition (4) in the extended Optional Stopping Time Theorem. Using Optional Stopping Theorem (Theorem~\ref{thm:Doob2}) on stopping time $T_k$ gives,}
                                                &\leq b^2 - (b-X_k)^2
                                                % \\ \intertext{Note that $\Ek{X_k}\leq E(X_k)\leq E(X_0)$,}
                                                % &\leq  b^2 - (b-X_0)^2
\end{align*}
So we can see the bounds for $\Ek{Y_{T_k}} \in [0,b^2 - (b-X_k)^2]$ and deduce that $b^2 - (b-X_k)^2 \geq \Ek{Y_k-Y_{T_k}} \geq 0$. By setting
\begin{align*}
    \theta_k:= \frac{\Ek{Y_k-Y_{T_k}}}{\delta}  \in [0,\frac{b^2}{\delta}],
\end{align*}
and using $\Ek{Y_k -Y_{T_k} } \geq \delta \Ek{T_k-k }$, we have shown the following estimate:
\begin{align}
      \E{ T_{k} - k \mid \mathcal{F}_{k}} \leq \theta_{k} \quad \text{for some $\theta_{k} \geq 0$.} \label{eq:khit_time}
\end{align}
Taking $\theta := eb^2 / \delta$, we get  for all natural numbers $k >0$,
\begin{align}
    \E{\mathds{1}_{ \{T_{k} - k > \theta \}}\mid \mathcal{F}_{k} } &= \Pr (T_{k} - k > \theta \mid \mathcal{F}_{t})\nonumber \\ \intertext{Since $T_k-k\geq 0$, we can apply Markov's inequality}
    &\leq  \E{ T_{k} - k \mid \mathcal{F}_{k}} / \theta \nonumber \\ \intertext{Using Eq.~(\ref{eq:khit_time}) gives}
                                                                                                           &\leq \theta_{k} / \theta \nonumber \\ 
                                                           & \leq e^{-1}. \label{eq:Markov2}
\end{align}
\par Next, we construct a recurrence relation by using Equation \ref{eq:Markov2} above. We consider the intersection decomposition (A similar intersection decomposition can be found in \cite{menshikov_popov_wade_2016})  between event $\{T_0>(k+1)\theta\}$ and $\{T_0>k\theta \}$ by introducing $T_{k\theta}$
\begin{align*}
    \{T_0 >(k+1)\theta\} &= \{{T_0} >k\theta \} \cap \{{T_{k\theta}} >(k+1)\theta \}  \\
                       &= \{T_0 >k\theta \} \cap \{{T_{k\theta}} - k\theta >\theta \} \\ \intertext{Then,}
                         \mathds{1}_{   \{T_0 >(k+1)\theta\}}  &=   \mathds{1}_{\{T_0 >k\theta \} } \cdot  \mathds{1}_{\{{T_{k\theta}} - k\theta >\theta \}}.
\end{align*}
So, we have the recurrence relation:
\begin{align*}
    \Pr(T_0 >(k+1) \theta) &= \E{ \mathds{1}_{\{{T_0 >(k+1)\theta \} } }} \\
                             &= \E{ \mathds{1}_{\{T_0 >k\theta \} } \cdot  \mathds{1}_{\{ T_{k\theta }  - k\theta >\theta \} } } \\ \intertext{note that $\{T_0 - k\theta >0 \} \in \mathcal{F}_{k\theta}$ and use the property of conditional expectation,}
                             &= \E{\mathds{1}_{\{T_0 >k\theta \} } \E{\mathds{1}_{\{{T_{k\theta}} - k\theta >\theta \}} \mid \mathcal{F}_{k\theta} }}\\ \intertext{Using Eq.~\ref{eq:Markov2} gives}
                             &\leq \E{\mathds{1}_{\{T_0 >k\theta \} } \cdot e^{-1}}\\
                             &= \Pr(T_0 >k\theta) \cdot e^{-1} \\ \intertext{By induction, we have }
                             &\leq e^{-(k+1)}.
\end{align*}
Then, finally, we derive (rearranging all the terms) for $\tau\geq 0$:
\begin{align*}
    \Pr(T_0\geq \tau) \leq e^{-\tau / \theta} 
\end{align*}
where $\theta \geq  e b^2/\delta $.  
\end{proof}

\SecThreeMainOneplus*
\begin{proof}[Proof of Theorem~\ref{thm:ExpTail}]
    Let us define $Y_t=n-X_t$ and $T:=\{t\geq 0 \mid X_t \geq n\}=\{t\geq 0 \mid Y_t \leq 0\}$. So from $(C2),(C3)$ we have 
    \begin{align*}
        \Et{Y_t-Y_{t+1}} &\geq 0 \\
        \Et{\left(Y_t-Y_{t+1} \right)^2} &\geq \delta
    \end{align*}
    This implies that $Y_t$ satisfies $(A1)$ in Theorem~\ref{thm:Uppertail2}. Notice that $Y_t \in [0,n]$ so we set $b=n$ in Theorem~\ref{thm:Uppertail2}, and thus we satisfy $(A2)$. We now have $Y_t$ satisfies all conditions in Theorem~\ref{thm:Uppertail2}.  We apply Theorem~\ref{thm:Uppertail2} to $Y_t$ with $Y_{T}=0$, we obtain the desired results.
\end{proof}

\SecThreeLemOne*
\begin{proof}[Proof of Lemma~\ref{lem:stepsize}]
Notice that for all $T>t$,
\begin{align*}
     E\bigl[|X_{t+1}-X_t|\mid {\mathcal F}_t ] & = \int_{0}^{\infty} \Pr (|X_{t+1}-X_t|\geq j\ \mid {\mathcal F}_t) dj    \\ \intertext{Using the step size condition, we have}
                                                            & \leq \int_{0}^{\infty} \frac{r}{(1+\eta)^j} dj \\
                                                            &=\frac{r}{\log (1+\eta )}:=c.
\end{align*}
\end{proof}

\SecThreeMainCorOne*
\begin{proof}[Proof of Theorem~\ref{cor:ExpTail_cor}]
    This proof is a direct result of Theorem~\ref{thm:ExpTail} by using fixed step size condition (C1).
\end{proof}

\SecThreeMainTwo*
\begin{proof}[Proof of Theorem~\ref{thm:ExpTail_2absorbing}] We follow the analysis of Theorem~\ref{thm:Uppertail2} by replacing $Y_t = X_t(n-X_t)$ and $Z_t=Y_t+\delta t$ in the proof. 
 For any $ k\in \mathbb{N}$, we define $T_{k} = \inf \{t \geq k \mid X_{t} \leq 0 \text{ or } X_t \geq n \}$. Note that for any $k$ and $m$, if $m<k$, then  $\{\omega\in\Omega \mid T_k(\omega)=m\}=\emptyset\in \mathcal{F}_m$ and if $m \geq k$, from definition of $T_k$ (consider the process after $X_k$), 
 $\{\omega\in\Omega\mid T_k(\omega)=m\}=\{\omega\in\Omega\mid X_i(\omega)\in (0,n) \text{ and } X_m(\omega)\leq 0, X_m(\omega)\geq n  \text{ for $i=k, \dots, m-1$} \}$ is in the $\sigma$-algebra $\mathcal{F}_m$ since $X$ is adapted to $\mathcal{F}$.Thus, $T_k$ is a stopping time with respect to $\left(\mathcal{F}_t\right)_{t\geq 0}$. 
 \par We would like to apply the extended optional stopping theorem for the random variable $X_t$. To do so, we first need to prove that the expectation of $T$ is finite and thus the expectation of $T_k$ is finite for any given $k$. We want to apply additive drift theorem~\cite{he_drift_2001} on $Y_t$ instead of directly on $X_t$.
 Note that if $X_t \leq 0$ or $X_t \geq n$, then $Y_t=X_t(n-X_t) \leq 0$. If $Y_t \leq 0$, then $X_t(n-X_t)\leq 0$ which implies that either $X_t \leq 0$ or $n-X_t \leq 0$. We can see the equivalence of these two events from here. The equivalence of these two events gives
\begin{align*}
    T=\inf\{t \geq 0 \mid X_t  \leq 0 \text{ or } X_t \geq n\} = \inf\{t \geq 0 \mid Y_t  \leq 0\}.
\end{align*}
Consider the process $Y_t = X_t(n-X_t)$ for $t \geq 0$, we have
\begin{align*}
    \Et{Y_{t}-Y_{t+1}} &= \Et{X_t(n-X_t)-X_{t+1}(n-X_{t+1} } \\
                       &= n \Et{X_t-X_{t+1}}+ \Et{X_{t+1}^2-X_t^2} \\ \intertext{Using condition (C3) and $\Et{X_{t+1}-X_t}=0$ gives}
                       &=\Et{\left(X_{t}-X_{t+1} \right)^2} \geq \delta .
\end{align*}        
And we have 
\begin{align*}
    \Et{Z_t-Z_{t+1}} = \Et{Y_t-Y_{t+1}-\delta} \geq 0 .
\end{align*}
Also notice that $Y_t, Z_t$ are adapted to the filtration $\mathcal{F}_t$ and due to $X_t$ is finite then $\E{|X_t|}<\infty$. This implies that $\E{|Y_t|}, \E{|Z_t|}$ are bounded as well for any given $t$. Then, both $Z_t$ and $Y_t$ are super-martingales according to Definition~\ref{def:martingale}. 
Since we have shown that $T_k$ is a stopping time, we can apply Theorem~\ref{thm:Doob2} with respect to $Z_t$. 
\begin{align*}
    \Ek{Z_k-Z_{T_k}} &\geq 0 \\ \intertext{Substituting $Z_t$ gives}
    \Ek{Y_k+\delta k -\left(Y_{T_k}+\delta T_k \right)} &\geq 0 \\ \intertext{Rearranging the equation gives}
    \Ek{Y_k -Y_{T_k} } &\geq \delta \Ek{T_k-k }
\end{align*}

Now, note that 
\begin{align*}
    0\leq \Ek{Y_{T_k}} &= \Ek{n \left(n-X_{T_k} \right)} \\
                       &= n^2 - n\Ek{X_{T_k}} \\ \intertext{Since $\E{T}<\infty$ as shown and we have shown Lemma~\ref{lem:stepsize}, we satisfy condition (4) in the extended Optional Stopping Time Theorem. Using Optional Stopping Theorem (Theorem~\ref{thm:Doob2}) on stopping time $T_k$ gives,}
                                                &\leq n^2 - nX_k
\end{align*}
So we can see the bounds for $\Ek{Y_{T_k}} \in [0,n^2 - nX_k]$ and deduce that $n^2 - nX_k \geq \Ek{Y_k-Y_{T_k}} \geq 0$. By setting
\begin{align*}
    \theta_k:= \frac{\Ek{Y_k-Y_{T_k}}}{\delta}  \in [0,\frac{n^2}{2\delta}],
\end{align*}
and using $\Ek{Y_k -Y_{T_k} } \geq \delta \Ek{T_k-k }$, we have shown the following estimate:
\begin{align}
      \E{ T_{k} - k \mid \mathcal{F}_{k}} \leq \theta_{k} \quad \text{for some $\theta_{k} \geq 0$.} \label{eq:khit_time2}
\end{align}

Taking $\theta := \frac{en^2}{2\delta}$, we get  for all natural numbers $k >0$,
\begin{align}
    \E{\mathds{1}_{ \{T_{k} - k > \theta \}}\mid \mathcal{F}_{k} } &= \Pr (T_{k} - k > \theta \mid \mathcal{F}_{t})\nonumber \\ \intertext{Since $T_k-k\geq 0$, we can apply Markov's inequality}
    &\leq  \E{ T_{k} - k \mid \mathcal{F}_{k}} / \theta \nonumber \\ \intertext{Using Eq.~(\ref{eq:khit_time2}) gives}
                                                                                                           &\leq \theta_{k} / \theta \nonumber \\ 
                                                           & \leq e^{-1} \label{eq:Markov3}
\end{align}
\par Next, we consider constructing a recurrence relation by using Equation \ref{eq:Markov3} above. Note that we consider the intersection decomposition (A similar intersection decomposition can be found in \cite{menshikov_popov_wade_2016})  between event $\{T_0>(k+1)\theta\}$ and $\{T_0>k\theta \}$ by introducing $T_{k\theta}$
\begin{align*}
    \{T_0 >(k+1)\theta\} &= \{{T_0} >k\theta \} \cap \{{T_{k\theta}} >(k+1)\theta \}  \\
                       &= \{T_0 >k\theta \} \cap \{{T_{k\theta}} - k\theta >\theta \} \\ \intertext{Then,}
                         \mathds{1}_{   \{T_0 >(k+1)\theta\}}  &=   \mathds{1}_{\{T_0 >k\theta \} } \cdot  \mathds{1}_{\{{T_{k\theta}} - k\theta >\theta \}}.
\end{align*}
So, we have the recurrence relation:
\begin{align*}
    \Pr(T_0 >(k+1) \theta) &= \E{ \mathds{1}_{\{{T_0 >(k+1)\theta \} } }} \\
                             &= \E{ \mathds{1}_{\{T_0 >k\theta \} } \cdot  \mathds{1}_{\{ T_{k\theta }  - k\theta >\theta \} } } \\ \intertext{note that $\{T_0 - k\theta >0 \} \in \mathcal{F}_{k\theta}$ and use the property of conditional expectation,}
                             &= \E{\mathds{1}_{\{T_0 >k\theta \} } \E{\mathds{1}_{\{{T_{k\theta}} - k\theta >\theta \}} \mid \mathcal{F}_{k\theta} }}\\ \intertext{Using Eq.~\ref{eq:Markov3} gives}
                             &\leq \E{\mathds{1}_{\{T_0 >k\theta \} } \cdot e^{-1}}\\
                             &= \Pr(T_0 >k\theta) \cdot e^{-1} \\ \intertext{By induction, we have }
                             &\leq e^{-(k+1)}.
\end{align*}
Then, finally, we derive (rearranging all the terms) for $\tau\geq 0$:
\begin{align*}
    \Pr(T_0\geq \tau) \leq e^{-\tau / \theta} 
\end{align*}
where $\theta = \frac{en^2}{2\delta} $.

\end{proof}

\SecThreeMainThree*
\begin{proof}[Proof of Theorem~\ref{thm:ExpTail_drift}]
We follow the analysis of Theorem~\ref{thm:Uppertail2} by By replacing $Y_t = X_t$ and $Z_t=Y_t-\varepsilon t$ in the proof. 
 For any $ k\in \mathbb{N}$, we define $T_{k} = \inf \{t \geq k \mid  X_t \geq n \}$. Note that for any $k$ and $m$, if $m<k$, then  $\{\omega\in\Omega \mid T_k(\omega)=m\}=\emptyset\in \mathcal{F}_m$ and if $m \geq k$, from definition of $T_k$ (consider the process after $X_k$), $\{\omega\in\Omega\mid T_k(\omega)=m\}=\{\omega\in\Omega\mid X_i(\omega)<n \text{ and } X_m(\omega)\geq n \text{ for $i=k, \dots, m-1$} \}$ is in the $\sigma$-algebra $F_m$ since $X$ is adapted to $\mathcal{F}$.Thus, $T_k$ is a stopping time with respect to $\left(\mathcal{F}_t\right)_{t\geq 0}$. We would like to apply the extended optional stopping theorem for the random variable $X_t$. To do so, we first need to prove that the expectation of $T$ is finite and thus the expectation of $T_k$ is finite for any given $k$. We apply additive drift theorem~\cite{he_drift_2001} on $X_t$ to get $\E{T} \leq \frac{n-X_0}{\varepsilon}$. And we have 
\begin{align*}
    \Et{Z_{t+1}-Z_t} = \Et{X_{t+1}-X_{t+1}-\varepsilon} \geq 0.
\end{align*}
Also notice that $Z_t$ is adapted to the filtration $\mathcal{F}_t$ and due to $X_t$ is finite then $\E{|X_t|}<\infty$. This implies that $\E{|Z_t|}$ are bounded as well for any given $t$. Then, both $Z_t$ and $Y_t$ are sub-martingales according to Definition~\ref{def:martingale}. 
Since we have shown that $T_k$ is a stopping time, we can apply Theorem~\ref{thm:Doob2} with respect to $Z_t$. 
\begin{align*}
    \Ek{Z_{T_k}-Z_k} &\geq 0. \\ \intertext{Substituting $Z_t$ gives}
    \Ek{X_{T_k}-\varepsilon T_k - \left(X_k-\varepsilon k \right)} &\geq 0. \\ \intertext{Rearranging the equation gives}
    \Ek{X_{T_k} -X_{k} } &\geq \varepsilon \Ek{T_k-k }.
\end{align*}

Now, Since $\E{T}<\infty$ as shown and we have shown Lemma~\ref{lem:stepsize}, we satisfy condition (4) in the extended Optional Stopping Time Theorem. Using Optional Stopping Theorem (Theorem~\ref{thm:Doob2}) on stopping time $T_k$ gives $\Ek{X_{T_k}} \geq X_k$. So we can see the bounds for $\Ek{X_{T_k}} \in [X_k,n]$ and deduce that $n - X_k \geq \Ek{X_{T_k}-X_k} \geq 0$. By setting
\begin{align*}
    \theta_k:= \frac{\Ek{X_{T_k}-X_k}}{\varepsilon}  \in [0,\frac{n}{\varepsilon}],
\end{align*}
and using $\Ek{X_{T_k} -X_{k} } \geq \varepsilon \Ek{T_k-k }$, we have shown the following estimate:
\begin{align}
      \E{ T_{k} - k \mid \mathcal{F}_{k}} \leq \theta_{k} \quad \text{for some $\theta_{k} \geq 0$.} \label{eq:khit_time4}
\end{align}

Taking $\theta := en / \varepsilon$, we get  for all natural numbers $k >0$,
\begin{align}
    \E{\mathds{1}_{ \{T_{k} - k > \theta \}}\mid \mathcal{F}_{k} } &= \Pr (T_{k} - k > \theta \mid \mathcal{F}_{t})\nonumber \\ \intertext{Since $T_k-k\geq 0$, we can apply Markov's inequality}
    &\leq  \E{ T_{k} - k \mid \mathcal{F}_{k}} / \theta \nonumber \\ \intertext{Using Eq.~(\ref{eq:khit_time4}) gives}
                                                                                                           &\leq \theta_{k} / \theta \nonumber \\ 
                                                           & \leq e^{-1} \label{eq:Markov4}
\end{align}
\par Next, we consider constructing a recurrence relation by using Equation \ref{eq:Markov4} above. Note that we consider the intersection decomposition (A similar intersection decomposition can be found in \cite{menshikov_popov_wade_2016})  between event $\{T_0>(k+1)\theta\}$ and $\{T_0>k\theta \}$ by introducing $T_{k\theta}$
\begin{align*}
    \{T_0 >(k+1)\theta\} &= \{{T_0} >k\theta \} \cap \{{T_{k\theta}} >(k+1)\theta \}  \\
                       &= \{T_0 >k\theta \} \cap \{{T_{k\theta}} - k\theta >\theta \} \\ \intertext{Then,}
                         \mathds{1}_{   \{T_0 >(k+1)\theta\}}  &=   \mathds{1}_{\{T_0 >k\theta \} } \cdot  \mathds{1}_{\{{T_{k\theta}} - k\theta >\theta \}}.
\end{align*}
So, we have the recurrence relation:
\begin{align*}
    \Pr(T_0 >(k+1) \theta) &= \E{ \mathds{1}_{\{{T_0 >(k+1)\theta \} } }} \\
                             &= \E{ \mathds{1}_{\{T_0 >k\theta \} } \cdot  \mathds{1}_{\{ T_{k\theta }  - k\theta >\theta \} } } \\ \intertext{note that $\{T_0 - k\theta >0 \} \in \mathcal{F}_{k\theta}$ and use the property of conditional expectation,}
                             &= \E{\mathds{1}_{\{T_0 >k\theta \} } \E{\mathds{1}_{\{{T_{k\theta}} - k\theta >\theta \}} \mid \mathcal{F}_{k\theta} }}\\ \intertext{Using Eq.~\ref{eq:Markov3} gives}
                             &\leq \E{\mathds{1}_{\{T_0 >k\theta \} } \cdot e^{-1}}\\
                             &= \Pr(T_0 >k\theta) \cdot e^{-1} \\ \intertext{By induction, we have }
                             &\leq e^{-(k+1)}.
\end{align*}
Then, finally, we derive (rearranging all the terms) for $\tau\geq 0$:
\begin{align*}
    \Pr(T_0\geq \tau) \leq e^{-\tau / \theta} 
\end{align*}
where $\theta = \frac{en}{\varepsilon} $.

\end{proof}

\SecFourMainOne*
\begin{proof}[Proof of Theorem~\ref{thm:2sat2}] 
This proof follows the same proof exactly in \cite{gobel2018intuitive} except we consider the tail bound for the runtime. We define $X_t \in [n]$ to be the number of variables that the current truth assignment agrees with a satisfying assignment and $T$ to be the first hitting time that $X_t=n$.
\par From the proof of Theorem~\ref{thm:2sat} in \cite{gobel2018intuitive}, it has been shown that
the variance bound $Var(X_{t+1}-X_{t} \mid \mathcal{F}_t) \geq 1$ and $\Pr(X_{t+1}=X_t+1\mid \mathcal{F}_t)\geq 1/2$ and $\Pr(X_{t+1}=X_t-1 \mid \mathcal{F}_t)\leq 1/2$. Thus, we can easily have $\E{(X_{t+1}-X_t)^2 \mid \mathcal{F}_t} \geq Var(X_{t+1}-X_{t} \mid \mathcal{F}_t) \geq 1 $ and $\E{X_{t+1}-X_{t}\mid \mathcal{F}_t}\geq 0$. This satisfies conditions (C2) and (C3) in Theorem~\ref{thm:ExpTail}.
Then we need to verify that the randomised 2-SAT algorithm satisfies the step size condition. It is easy to see that from Algorithm~\ref{alg:2Sat}, $X_t$ always has a fixed step size $1$ (either backwards or forwards). Then, we verified the step size condition and then applied Theorem~\ref{thm:ExpTail} with $\theta = en^2$ to obtain the desired tail bound.
\begin{align*}
    \Pr(T\geq r n^2) = e^{-r n^2/ en^2} \leq  e^{-r/e}.
\end{align*}
The rest of the step follows from the proof of Theorem~\ref{thm:2sat} in \cite{gobel2018intuitive} and by multiplying the executed time of the algorithm (i.e. $O(n^2)$), we can conclude the desired tail bound.
\end{proof}

\SecFourMainTwo*
\begin{proof}[Proof of Theorem~\ref{thm:recoloring}] 
This proof follows the same proof exactly in \cite{gobel2018intuitive} except we consider the tail bound for the runtime. 
We first fix a 3-colouring of this given graph, two colours out of three and a set of vertices which consists of these two colours. 
We define the size of this set of vertices by $m$ and $Y_t \in [m]$ to be the number of vertices in this set, whose colour in the 2-colouring provided by the Recolour algorithm at iteration $t \in \mathbb{N}$ matches their colour in the original 3-colouring, with the condition that these matching colours are one of the two fixed colours. We define the first hitting time by $T:=\inf \{t>0 \mid Y_t= 0 \text{ or } m\}$.
\par It has been shown in the proof of Theorem~\ref{thm:McDiarmid} in \cite{gobel2018intuitive} that 
``the runtime of Recolour is bounded from
above by the time that Recolour takes to find such a colouring."  Thus, we consider the tail bound for $T$. It has also been shown in \cite{gobel2018intuitive} that for $s \in [m] \setminus \{0,m\}$,  
\begin{align*}
    \Pr \left(Y_{t+1}=Y_t \pm 1 \mid Y_t =s \right) &= 1/3  \\ \intertext{Note that we can compute the drift and the second moment of the drift}
    \E{Y_t-Y_{t+1} \mid Y_t =s } &= 1 \times \frac{1}{3} + (-1) \times \frac{1}{3} + 0 \times \frac{1}{3} \\
                                &=0 \\
    \E{ \left(Y_t-Y_{t+1} \right)^2 \mid Y_t =s } &=1 \times \frac{1}{3} + (-1)^2 \times \frac{1}{3}  = \frac{2}{3}
\end{align*}
Now, we satisfy condition (C1*), (C3) and zero drift condition in Theorem~\ref{thm:ExpTail_2absorbing}. 
Notice that $Y_0(m-Y_0)\leq m^2/4$ since $Y_0 \in [m]$ and this attains the maximum when $Y_0=m/2$. Also, as the definition of $m$, we have $m\leq n$. Thus, we derive the following
\begin{align*}
    \E{T} &\leq \frac{3\E{Y_0\left(m-Y_0\right)}}{2} \leq \frac{3n^2}{8}. \\ \intertext{Moreover, for any $r\geq 0$, we have}
    \Pr \left(T \geq r n^2 \right)  &\leq  e^{- 2r n^2 \cdot \frac{2}{3}/ en^2} \leq e^{\frac{-4r}{3e}} .
\end{align*}
The rest of the step follows from the proof of Theorem~\ref{thm:McDiarmid} in \cite{gobel2018intuitive} and by multiplying the executed time of the algorithm at each step (i.e. $O(n^2)$), so we finish the proof.
\end{proof}

\SecFiveMainOne*
\begin{proof}[Proof of Theorem~\ref{thm:RLS_runtime}]
This proof follows the same proof exactly in \cite{hevia2023runtime} except we consider the tail bound for the runtime.
We define $T:=\inf\{t>0 \mid(x_t,y_t)\in \OPT\}=\inf\{t>0 \mid M_t=0\}$, where $(x_t,y_t)$ are the current solutions of RLS-PD and $0\leq M_t \leq n(\alpha+\beta) \leq 2n$ is defined in Definition~\ref{def:forward_backward_drift}. 
Let $b=2(A+B)\sqrt{n}+1$. 
It has been shown in the proof of Theorem 3.1 in \cite{hevia2023runtime} that 
for every generation $t<T$, the drift of a new potential function 
\begin{align*}
    h(\M_t)=\begin{cases}
        b^2-(b-\M_t)^2 & \text{if } \M_t \le 2(A+B)\sqrt{n} \\
        b^2-(b-\M_t) & \text{if } \M_t > 2(A+B)\sqrt{n}
    \end{cases}
\end{align*}
is lower bounded by the function
\begin{align*}
    \delta(\M_t)
    &=\begin{cases}
        \delta_1 & \text{if } \M_t \le 2(A+B)\sqrt{n} \\
        \frac{\M_{t}-(A+B)\sqrt{n}}{2n} & \text{if } \M_t > 2(A+B)\sqrt{n}
    \end{cases} \\
    &\geq \frac{1}{2\sqrt{n}}:=\delta_1.
\end{align*}
Due to the piece-wise drift, we cannot directly apply our variance drift theorem (tail bound). To derive the exponential tail bound for the runtime, we divide the analysis into two phases.
We define $T_{\text{phase1}}:= \inf \{t>0 \mid \M_t < b \text{ given } \M_0 \geq b\}$.
Note that $T=\inf \{t>0 \mid \M_t=0\}=\inf \{t \geq  T_{\text{phase1}} \mid \M_t=0\}$. So we define $T_{\text{phase2}}:=T-T_{\text{phase1}}$. From the definition, we can bound $T_{\text{phase2}}$ from above by the first hitting $T_{\text{phase2}}':=\inf \{t >0 \mid \M_t=0 \text{ given } \M_0=b-1\}$. From the drift condition above for the case $\M_t \leq 2(A+B)\sqrt{n}$, we satisfy (A1), (A2) in Theorem~\ref{thm:Uppertail2} and the step size condition directly follows from the fact that RLS-PD only makes one step jump at each iteration. 
By applying Theorem~\ref{thm:Uppertail2}, we get: for any $r>0$,
\begin{align*}
    \Pr(T_{\text{phase2}}' \geq rn^{1.5}) 
    &\leq e^{-rn^{1.5} \delta_1 / e(b-1)^2} \\ \intertext{Taking $\delta_1=\frac{1}{2\sqrt{n}}$ and substituting $b=2\left(A+B \right)\sqrt{n}+1$ give}
    & \leq  \exp \left(-\frac{rn^{1.5} \frac{1}{2\sqrt{n}}}{4e(A+B)^2 n} \right) \\
    & \leq \exp \left(-\frac{r}{8e(A+B)^2} \right) = e^{-\Omega(r)}.
\end{align*}
Then, we consider the tail bound for $T_{\text{phase1}}$.
Let $h_{\max}:=b^2-b+n(\alpha+\beta)=O(n)$.
We use Theorem~\ref{thm:ExpTail_drift} by setting $X_t=h_{\max}-h(\M_t)$ for $t<T_{\text{phase1}}$. The drift for $X_t$ we obtain from above is at least $\delta_1$. So we have the tail bound
\begin{align*}
    \Pr(T_{\text{phase1}} > rn^{1.5}) 
    &\leq e^{-rn^{1.5} \delta_1 / eh_{\max}} \\ \intertext{Taking $\delta_1=\frac{1}{2\sqrt{n}}$ and substituting $h_{\max}=b^2-b+n(\alpha+\beta)$ give}
    & \leq  \exp \left(-\frac{rn^{1.5} \frac{1}{2\sqrt{n}}}{e\left(b^2-b+n(\alpha+\beta) \right)} \right) \\
    & = e^{-\Omega(r)}.
\end{align*}
Note that $\{T \geq 2rn^{1.5}\} \subseteq \{T_{\text{phase1}} \geq rn^{1.5}\}\cup \{T_{\text{phase2}} \geq rn^{1.5}\}$. Using the union bound gives
\begin{align*}
    \Pr \left(T \geq 2rn^{1.5} \right) 
    &\leq \Pr \left(T_{\text{phase1}} \geq rn^{1.5} \right) +
    \Pr \left(T_{\text{phase2}} \geq rn^{1.5} \right) \\ \intertext{Note that $T_{\text{phase2}}' \geq T_{\text{phase2}}$ and thus we have $\{T_{\text{phase2}} \geq rn^{1.5}\} \subseteq \{T_{\text{phase2}}' \geq rn^{1.5}\}$. Substituting the bounds gives}
    &\leq 2 e^{-\Omega(r)} = e^{-\Omega(r)}.
\end{align*}
We complete the proof.

\end{proof}

\SecFiveMainTwo*
\begin{proof}[Proof of Theorem~\ref{thm:reaching_the_optimum_RLS}]
This proof follows the same proof exactly in \cite{hevia2023runtime} except we consider the tail bound for the runtime.
Define $T:=\inf\{t\mid\M_t \geq (A+B)\sqrt{n}\}$, where $M_t$ is the current Manhattan distance to the set $\OPT$.  We assume $M_0\le(A+B)\sqrt{n}$. Otherwise, we reduce to the trivial case with $T=O(1)$.
We would like to show $\E{T} \leq O(n)$.
It has been shown in the proofs of Lemma~\ref{lem:drift_Manhattan_RLS} in \cite{hevia2023runtime} that 
for every generation $t<T$,
\begin{align*}
    \E{\M_{t}-\M_{t+1}- \frac{\M_{t}-(A+B)\sqrt{n}}{2n};t<T \mid M_t} \ge 0.
\end{align*}
\par We cannot directly use additive drift on $Y_t:=(A+B)\sqrt{n}-\M_t$ since the drift $\Et{Y_t-Y_{t+1}}$ can potentially be negative. Instead of using additive drift, we use the variance drift theorem (Theorem~\ref{thm:Uppertail2}) to show the runtime with the tail bound. Let $a=0, b=(A+B)\sqrt{n}$ and $Y_t$ as defined above in Theorem~\ref{thm:Uppertail2}. 
As shown in the proof of Theorem 4.1 in \cite{hevia2023runtime},
\begin{align*}
    &\Et{(Y_{t+1}-Y_t)^2-2(Y_{t+1}-Y_t)(b-Y_t);t<T}\\
    &\ge \frac{1}{2}-\frac{2(A+B)}{\sqrt{n}} \\ \intertext{For sufficiently large $n$, there exists a constant $\delta_2>0$ s.t.}
    &\ge \delta_2.
\end{align*}
 The expected runtime is $\E{T}=O(b^2)=O(n)$ and
 tail bound with the worst case that $Y_0=b$ and $\tau=rn $, we get: for any $r>0$,
\begin{align*}
    \Pr \left(T > rn\right) 
    &\leq \exp \left(-\frac{rn \delta_2}{eb^2}\right) \\ \intertext{Taking constant $\delta_2>0$ and substituting $b$ give}
    & = \exp \left( -\frac{rn \cdot \delta_2}{(A+B)^2n} \right)  \\
    & = e^{-\Omega(r)}.
\end{align*}
\end{proof}

\SecSixLemOne*

\begin{proof}[Proof of Lemma~\ref{lem:RegretLemOne}]
We say random variable $Y$ stochastic dominates $X$ if for any $s \in \mathbb{R}$, 
\begin{align*}
    \Pr \left( X \geq s \right)\leq \Pr \left(Y \geq s \right).
\end{align*}
Let us substitute the probability distribution function for Geometric distribution into each random variable. 
\begin{align*}
    \Pr \left(X \geq s \right)
    &= \sum_{k=\floor{s}+1}^{\infty} (1-p)^k p \\ \intertext{Using the sum for geometric series gives}
    &= p \cdot \frac{(1-p)^{\floor{s}+1}\left(1-(1-p)^{\infty} \right)}{1- (1-p)}               \\
    &=(1-p)^{\floor{s}+1} 
\end{align*}
Similarly, we can obtain $\Pr \left(Y \geq s \right)=(1-q)^{\floor{s}+1} $. Since $p \geq q$, then $(1-p)\leq (1-q)$. Thus, for any $s \in \mathbb{R}$, we can derive $ \Pr \left( X \geq s \right)\leq \Pr \left(Y \geq s \right)$.

\end{proof}

\SecSixMain*

\begin{proof}[Proof of Theorem~\ref{thm:RegretMain}]
    This proof partially follows from the proof in \cite{larcher2023simple} since we need to deal with the tail bound of the product of multiple random variables.
    We use era/sub-era defined in Definition~\ref{def:era} and swap/mistakes defined in Definition~\ref{def:swap}.
    % Note an era denotes the time between two consecutive changes in the reward distribution. A sub-era is a (maximal)
    % interval in which both the underlying reward distribution and the arm $a^+$ remain constant. 
    We have $L$ eras since there are $L$ changes from the problem setting, and we let $M$ denote the number of such sub-eras. 
    % \cite{larcher2023simple} has already shown
    % $E[M] \leq 6L$. 
    \par Following the analysis of \cite{larcher2023simple}, we divide the proof into four parts. We denote the action value $S$ by $S_t$ at iteration $t$ in Algorithm~\ref{alg:rwab_challenge}. Define $\tau_{+1}$ as the hitting time of $+1$ in Algorithm~\ref{alg:rwab_challenge}, i.e., the minimal $t \geq 1$ such that $S_t \geq 1$. Similarly, for $s=\sqrt{T/L}$, define $\tau_{-s}$ to be the minimal $t \geq 1$ such that $S_t \leq s $. Denote the difference of rewards per iteration by $R_t = r^+-r^- \in [-1,1]$ and the expected regret is  $\Delta= |\E{R_t}| \in (0,1]$ (the expectation over the reward distributions). Since two reward distributions are fixed from the problem settings, we can see $\Delta$ is some constant in $(0,1]$.
    \par We denote the initial position by $S_0$ and $S_t \in (-s,1)$ for all $t<\min\{\tau_{-s},\tau_{+1}\}$. To simplify the calculation, we overestimate the regret per iteration simply by $R_t\leq 1$. We also define the accumulated regret for each case denoted by $\mathcal{R}_i$ for all $i \in [4]$ in Section 12.3: Supplementary material for the analysis of \rwab algorithm.
    \begin{itemize}
        \item[(1)] We firstly estimate $\mathcal{R}_1$. In this case, we accumulate regret only when running \Challenge with $a^+=a^*$. 
        \par The \Challenge breaks if either $S_t$ hits $+1$ or $-s$. To simplify the calculation, we only estimate the time when $S_t$ hits $+1$ in the following analysis.
        By applying Theorem~\ref{thm:ExpTail_drift} to $S_t$ with
        $\Et{S_{t+1}-S_t}\geq \Delta>0$, we have $\tau_{+1}$ is at most $\frac{\delta(1-S_0)}{\Delta }$ with probability at least $1-\exp \left(-\frac{\delta(1-S_0)}{e \cdot 1} \right)\geq 1-e^{-\delta/e}$ for any $\delta>0$. In other words, a \Challenge lasts at most $\frac{\delta(1-S_0)}{\Delta }$ with probability at least $1-e^{-\delta/e}$ for any $\delta>0$. Note that the regret per \Challenge is $1 \cdot \frac{\delta(1-S_0)}{\Delta }=\frac{\delta(1-S_0)}{\Delta }$. From Algorithm~\ref{alg:rwab}, we start a \Challenge with probability $\sqrt{\frac{L}{T}}$. We denote the number of \Challenge by $Z$ and $Z$ is subject to a Binomial distribution $\text{Bin}(T,\sqrt{\frac{L}{T}})$. By Chernoff Bound, we obtain for any $\delta>0$, 
        \begin{align}
            \Pr \left(Z \geq (1+\delta)\sqrt{LT} \right) \leq e^{-\delta^2\sqrt{LT}/(2+\delta)}. \label{eq:Chernoff}
        \end{align}    
        % Then, the expected number of \Challenge is $T\cdot \sqrt{\frac{L}{T}} =\sqrt{LT}$. 
        Using Union bound with $\{\mathcal{R}_1 \geq \delta(1+\delta)(1-S_0)\sqrt{LT}\}\subseteq \{Z \geq (1+\delta)\sqrt{LT} \}\cup \{\tau_{+1} \geq \frac{\delta(1-S_0)}{\Delta}\}$, we can derive
        \begin{equation}
            \Pr \left(\mathcal{R}_1 \leq \delta(1+\delta)\frac{(1-S_0)}{\Delta} \sqrt{LT}\right) \geq  1-e^{-\frac{\delta}{e}}-
            e^{-\frac{\delta^2\sqrt{LT}}{2+\delta}} \label{eq:R1}.
        \end{equation}

        \item[(2)] Next, we estimate $\mathcal{R}_2$. 
        % With $\mathcal{R}_2$ we denote the total regret accumulated in a change that occurs during an ongoing \Challenge until the end of \Challenge or a new change occurs. 
        Recall an era is defined in Definition~\ref{def:era}. As \Challenge breaks if the $S_t$ hits either $+1$ or $-s$. To simplify the calculation, we only estimate the time when $S_t$ hits $-s$ in the following analysis. By applying Theorem~\ref{thm:ExpTail_drift} to $X_t=s+S_t$ (for the case that $\tau_{-s}<\tau_{+1}$). The time a \Challenge breaks is, at most $\frac{ (s+S_0)\delta}{\Delta}$ with probability at least $1-\exp \left(-\frac{((s+S_0))\delta}{es} \right) \geq 1-e^{-\delta/e}$ for any $\delta>0$. Note that in each era, the regret accumulated is at most $1 \cdot \frac{\delta(s+S_0)}{\Delta}=\frac{\delta(s+S_0)}{\Delta}$. And we have $L$ eras as defined. So we derive with $S_0 \in (-s,1)$ and $s=\sqrt{T/L}$, 
       \begin{equation}
            \Pr \left(\mathcal{R}_2 \leq \frac{\delta(s+S_0)}{\Delta}L  \right) 
            % &=\Pr \left(\mathcal{R}_2 \leq r \sqrt{LT} \right) \\
            \geq 1- e^{-\delta/e}-e^{-\frac{\delta^2\sqrt{LT}}{2+\delta}} \label{eq:R2}.
       \end{equation}

        \item[(3)] We then estimate $\mathcal{R}_3$.
        % With $\mathcal{R}_3$ we denote the total regret accumulated in sub-eras with $a^+ \neq a^*$ while no \Challenge is running. 
        Recall the sub-era, is defined in Definition~\ref{def:era}. 
        To simplify our analysis, we overestimate $\mathcal{R}_3$ by assuming we accumulate regret at most $1$ at each step during this phase.
        % no change occurs until a new \Challenge is active. Otherwise, if there is a change occurs, then $a^+ = a^*$ and thus it contributes zero regrets until the next change occurs or a new \Challenge is active.
        % Above all, the only possibility that this \Challenge to \textbf{not} last until the end of the sub-era is to hit $+1$ before $-s$.
        % Then, we focus on \Challenge hits $-s$ before $+1$, and this case can end the sub-era. 
        We define $K$ to be the number of such steps in which we start a new \Challenge that ends the sub-era when no \Challenge is active. 
        We intend to show for any $\delta>0$, we have
        \begin{align}
            \Pr \left(K \geq \frac{12\delta\sqrt{T/L} }{\Delta} \right) \leq e^{-\Omega(\delta)}. \label{eq:number_challenge}
        \end{align}
        To see Eq.~(\ref{eq:number_challenge}) holds, note that by Lemma 1.1 (ii) in \cite{larcher2023simple}, each \Challenge has probability at least $\frac{\Delta}{12}$ if reaching $-s$ before $+1$ and thus end the sub-era.
        Recall that \rwab runs \Challenge with probability $\sqrt{L/T}$ at each step. So, 
        the probability of starting a new \Challenge that ends the sub-era, when no \Challenge is active, is at least $p:=\frac{\Delta}{12}\sqrt{L/T}$. We can see $K \sim \text{Geo}(p_K)$ where $p_K \geq p$. 
        We define $K'\sim \text{Geo}(p)$. By Lemma~\ref{lem:RegretLemOne}, we have $K' \succeq K$, and thus $\E{K}\leq \E{K'}=1/p $. We compute the following by using the fact that $K' \succeq K$. For any $r>0$,
        \begin{align*}
            \Pr \left(K \geq \frac{r}{p} \right) 
            &\leq \Pr \left(K' \geq \frac{r}{p} \right) \\ \intertext{We set $q:=1-p \in (0,1)$ and $K'\sim \text{Geo}(p)$.}
            & \leq \sum_{k=r/p}^{\infty} q^k p \intertext{Geometric series gives}
            &= p \cdot \frac{q^{r/p} (1-q^{\infty})}{1-q} \\ \intertext{$q:=1-p \in (0,1)$ implies that $q^{\infty}=0$ and this gives}
            &=q^{r/p} \\ \intertext{Using $q^{\alpha}=e^{\alpha\ln \left(q \right)}$ gives}
            &=e ^{r\ln (q) /p} \\
            &=\exp \left(-\frac{\ln\left(\frac{1}{1-p} \right)}{p}r \right) \\ \intertext{We define $f(p):=\frac{\ln\left(\frac{1}{1-p} \right)}{p}$ and note that $f(p)>1$ for any $p \in (0,1)$. Thus, we can conclude that}
            &\leq e^{-\delta}.
        \end{align*}
        So we prove Eq.~(\ref{eq:number_challenge}). Then, 
        the number of such steps $K$ is thus at most $\frac{12\delta \sqrt{T/L}}{\Delta}$ with probability $1-e^{-\delta}$, which is exactly Eq.~(\ref{eq:number_challenge}). So the regret contributing to $\mathcal{R}_3$ is bounded by $1\cdot \frac{12\delta \sqrt{T/L}}{\Delta}$ for each sub-era with probability at least $1-e^{-\delta}$. 
        \par For the number of sub-era $M$, from Lemma~2.1 in \cite{larcher2023simple}, we have $M\leq 2L+2N$ where $L$ is the number of changes and $N$ is the number of mistakes\footnote{This is defined in Definition~\ref{def:swap}.}. 
        Note that at each step there is a probability $\sqrt{L/T}$ of starting a \Challenge and by Lemma~1.1(i) in \cite{larcher2023simple},  ``this \Challenge has probability $2\sqrt{L/T}$ of ending with a mistake''. Thus, $N$ is stochastic dominated by $\text{Bin}(T,\sqrt{L/T}\cdot 2\sqrt{L/T})=\text{Bin}(T,\frac{2L}{T})$. By Chernoff bound, we have for any $\delta>0$,
        \begin{align*}
            \Pr \left(N \geq (1+\delta) 2L \right) \leq e^{-2L\delta^2/(2+\delta)} 
        \end{align*}
        Note that $M\leq 2L+2N$ implies that  $\{N\leq(1+\delta) 2L \} \subseteq \{M\leq 2L+4L(1+\delta)\}$. We deduce that 
        \begin{align}
            \Pr\left(M\leq  2L+4L(1+\delta)\right) 
            &\geq   \Pr\left(N\leq  (1+\delta) 2L\right) \nonumber\\
            &\geq 1- e^{-2L\delta^2/(2+\delta)}
            \label{eq:Chernoff2}
        \end{align}
        % We can bound the number of \Challenge that ends the sub-era trivially by the total number of \Challenge, namely $Z$. Using Eq.~(\ref{eq:Chernoff}) and Eq.~(\ref{eq:number_challenge}),
        % Recall that the number sub-era $M$ satisfies $\E{M} \leq 6L$. 
        Using $\mathcal{R}_3=M\cdot \text{Regret/per sub-era}$, we can obtain
        \begin{align*}
            % &\Pr \left(\mathcal{R}_3 \leq M \cdot 12r \sqrt{T/L}\right) \\
            &\Pr \left(\mathcal{R}_3 \geq  2L(3+2\delta) \cdot \frac{12\delta \sqrt{T/L}}{\Delta}\right) \\ \intertext{Notice that $\{\mathcal{R}_3 \geq  2L(3+2\delta) \cdot 12\delta \sqrt{T/L}\} \subseteq \{M\geq 2L(3+2\delta)\}\cup \{\text{Regret/per sub-era} \geq\frac{12\delta \sqrt{T/L}}{\Delta}\}$. We denote \text{Regret/per sub-era} by $Q$.
            Using Union bound gives}
            \leq & \Pr \left(M\geq 2L(3+2\delta) \right) \\ 
             \quad & + \Pr \left(Q \geq \frac{12\delta \sqrt{T/L}}{\Delta} \right) \\ \intertext{Using Eq.~(\ref{eq:number_challenge}) and Eq.~(\ref{eq:Chernoff2}) gives}
         \leq& e^{-2L\delta^2/(2+\delta)}+e^{-\delta}.
        \end{align*}
        In other words, we obtain
        \begin{equation}            
            \Pr \left(\mathcal{R}_3 \leq 24\delta(3+2\delta)\frac{ \sqrt{LT}}{\Delta}  \right) 
            \geq 1-e^{-2L\delta^2/(2+\delta)}-e^{-\delta} \label{eq:R3}.
        \end{equation}

        \item[(4)] We finally estimate $\mathcal{R}_4$.
        In this case, 
        ``each sub-era with $a^+ \neq a^*$ consists of several \Challenge which hit the $+1$, but one hits $-s$ to end the sub-era. Since this happens, $a^+$ is swapped, and then by definition, the sub-era ends \cite{larcher2023simple}''. Consider the steps in which a random walk with negative drift $-\Delta$ towards $-s$ is reset to $0$ whenever it goes above $+1$. By applying Theorem~\ref{thm:ExpTail_drift} to estimate the $\tau_{-s}$ with drift $\Delta$ towards $-s$, we have the total number of steps spent inside each sub-era 
        is at most $\frac{\delta(s+S_0)}{\Delta}$ with probability $1-e^{-\delta(s+S_0)/e}$ for any $\delta>0$. Since each step costs at most $1$, the regret is at most $\frac{\delta(s+S_0)}{\Delta}$ with probability $1-e^{-\delta(s+S_0)/e}$. 
        Recall that we derive the tail bound for the number of sub-era $M$ in Eq.~(\ref{eq:Chernoff2}) and that $\mathcal{R}_4=M\cdot \text{Regret/per sub-era}$.
        Using Union bound with $\{\mathcal{R}_4 \geq 2\delta(3+2\delta)(s+S_0)L\}\subseteq \{M \geq (3+2\delta)2L \}\cup \{\text{Regret/per sub-era} \geq \frac{\delta(s+S_0)}{\Delta}\}$ and Eq.~(\ref{eq:Chernoff2}), we obtain
        \begin{equation}            
            \Pr \left(\mathcal{R}_4 \leq 2\delta(3+2\delta)\frac{(s+S_0)}{\Delta}L\right) 
            \geq 1-e^{-\frac{2L\delta^2}{2+\delta}}-e^{-\frac{\delta(s+S_0)}{e}} \label{eq:R4}.
        \end{equation}
        By taking $\delta=\sqrt{\Delta \varepsilon}$ for any $\varepsilon \geq 1$ in Eq.~(\ref{eq:R1}), we can deduce that $\delta(1+\delta)=\delta+\delta^2\leq 2\varepsilon$ since $\Delta$ is some constant in $(0,1]$. Thus, we can obtain
        \begin{align*}
            \Pr \left(\mathcal{R}_1 \leq  2 \varepsilon \sqrt{LT} \right) &\geq 1 - 2e^{-\sqrt{\varepsilon}/e} \\ \intertext{Using $\delta=\sqrt{\Delta\varepsilon}\leq \varepsilon$ for $\varepsilon\geq 1$ in Eq.~(\ref{eq:R2}) gives}
            \Pr \left(\mathcal{R}_2 \leq  \varepsilon(L+ \sqrt{LT}) \right) &\geq 1 - 2e^{-\sqrt{\varepsilon}/e} \\ \intertext{Using $\delta(3+2\delta)=3\delta+2\delta^2 \leq 5\varepsilon $ in Eq.~(\ref{eq:R3}) gives}
            \Pr \left(\mathcal{R}_3 \leq 120 \varepsilon \sqrt{LT} \right) &\geq 1 - 2e^{-\sqrt{\varepsilon}} \\ \intertext{Using $2\delta(3+2\delta)=6\delta+4\delta^2 \leq 10\varepsilon $ in Eq.~(\ref{eq:R4}) gives}
            \Pr \left(\mathcal{R}_4 \leq  10 \varepsilon (L+ \sqrt{LT}) \right) &\geq 1 - 2e^{-\sqrt{\varepsilon}/e}. \\
        \end{align*}
        Above all, the total regret $\mathcal{R}$ consists of $\mathcal{R}_1, \mathcal{R}_2, \mathcal{R}_3 ,\mathcal{R}_4$. We define the following event: 
        \begin{align*}
            E &:= \{\mathcal{R} \leq  480 \varepsilon (L+  \sqrt{LT})\} \\
              &= \{\sum_{i=1}^4 \mathcal{R}_i \leq  480 \varepsilon (L+  \sqrt{LT})\}
        \end{align*}
        Considering the complement of Event $E$ and using the tail bound for each $\mathcal{R}_i$ give
        \begin{align*}
            &\Pr \left(\mathcal{R} > 480 \varepsilon (L+  \sqrt{LT}) \right) \\
            \leq & \Pr \left(\text{At least one $\mathcal{R}_i > \frac{480 \varepsilon (L+  \sqrt{LT})}{4}$} \right) \\ \intertext{Using tail bound for regret of each stage, we obtain}
            \leq & \max \{2e^{-\sqrt{\varepsilon}/e}, 2e^{-\sqrt{\varepsilon}}\} \\ 
            = &  2e^{-\sqrt{\varepsilon}/e}.
        \end{align*}
    \end{itemize}
\end{proof}

\section{Figures and Tables}
We defer all the figures and tables in the final section.
\begin{figure*}[htbp]
    \centering
  \subfloat[$\alpha=0.5$, $\beta=0.5$\label{1a}]{%
       \includegraphics[width=0.33\linewidth]{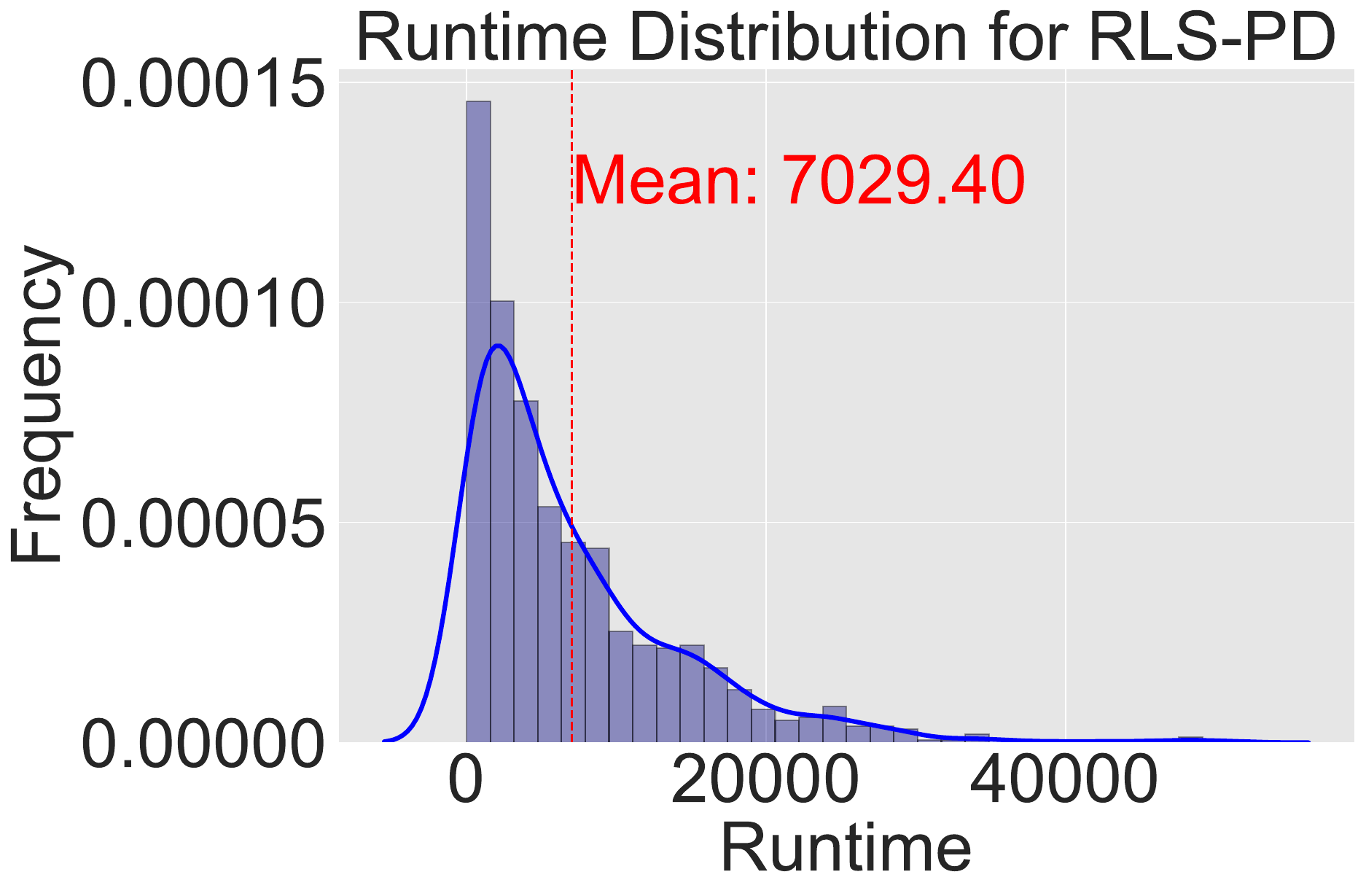}}
    \hfill
  \subfloat[$\alpha=0.3$, $\beta=0.3$\label{1b}]{%
        \includegraphics[width=0.33\linewidth]{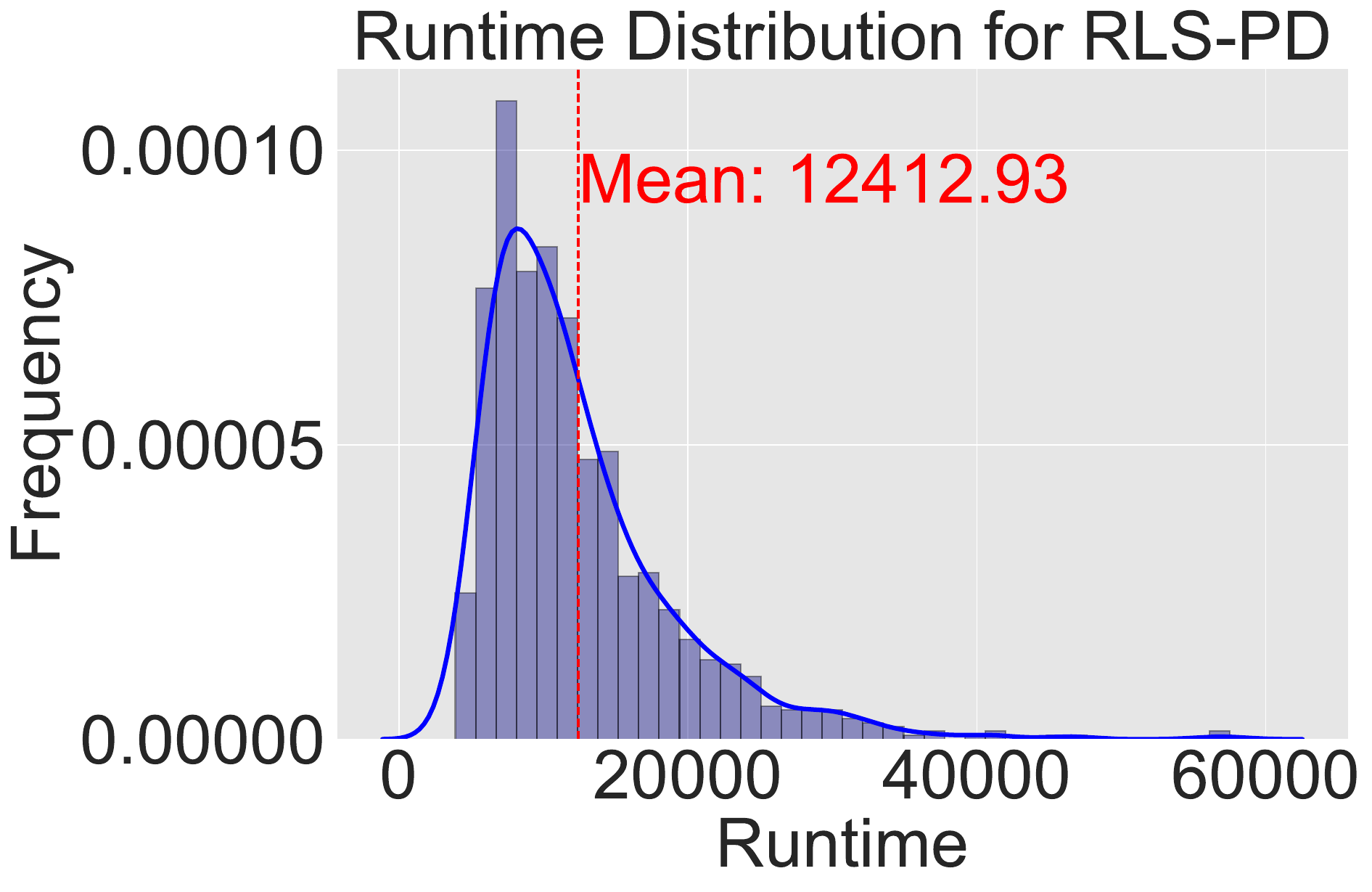}}
    \hfill
  \subfloat[$\alpha=0.7$, $\beta=0.7$\label{1c}]{%
        \includegraphics[width=0.305\linewidth]{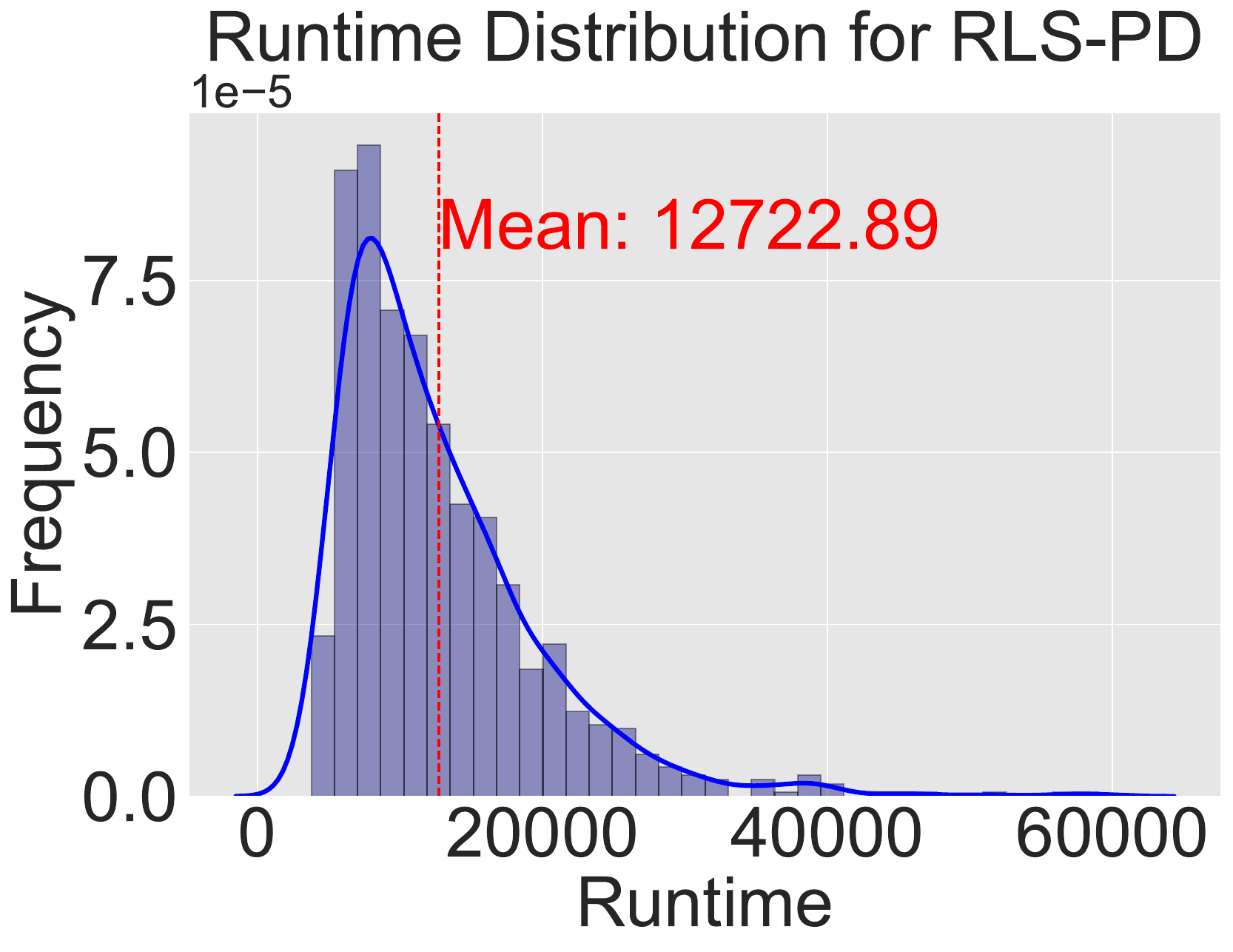}}
    \hfill
  \subfloat[$\alpha=0.3$, $\beta=0.7$\label{1d}]{%
        \includegraphics[width=0.33\linewidth]{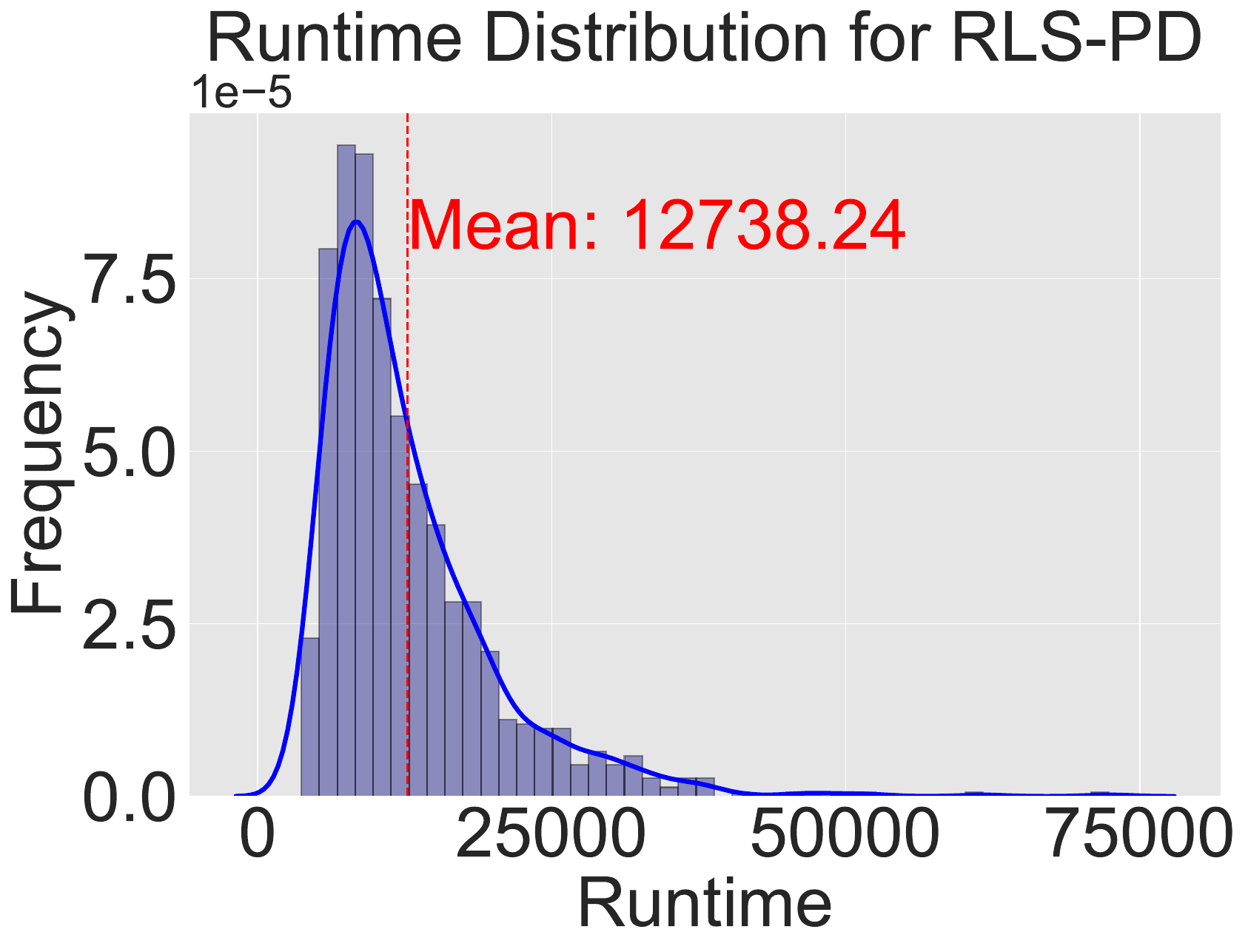}}
  \subfloat[$\alpha=0.7$, $\beta=0.3$\label{1e}]{%
        \includegraphics[width=0.33 \linewidth]{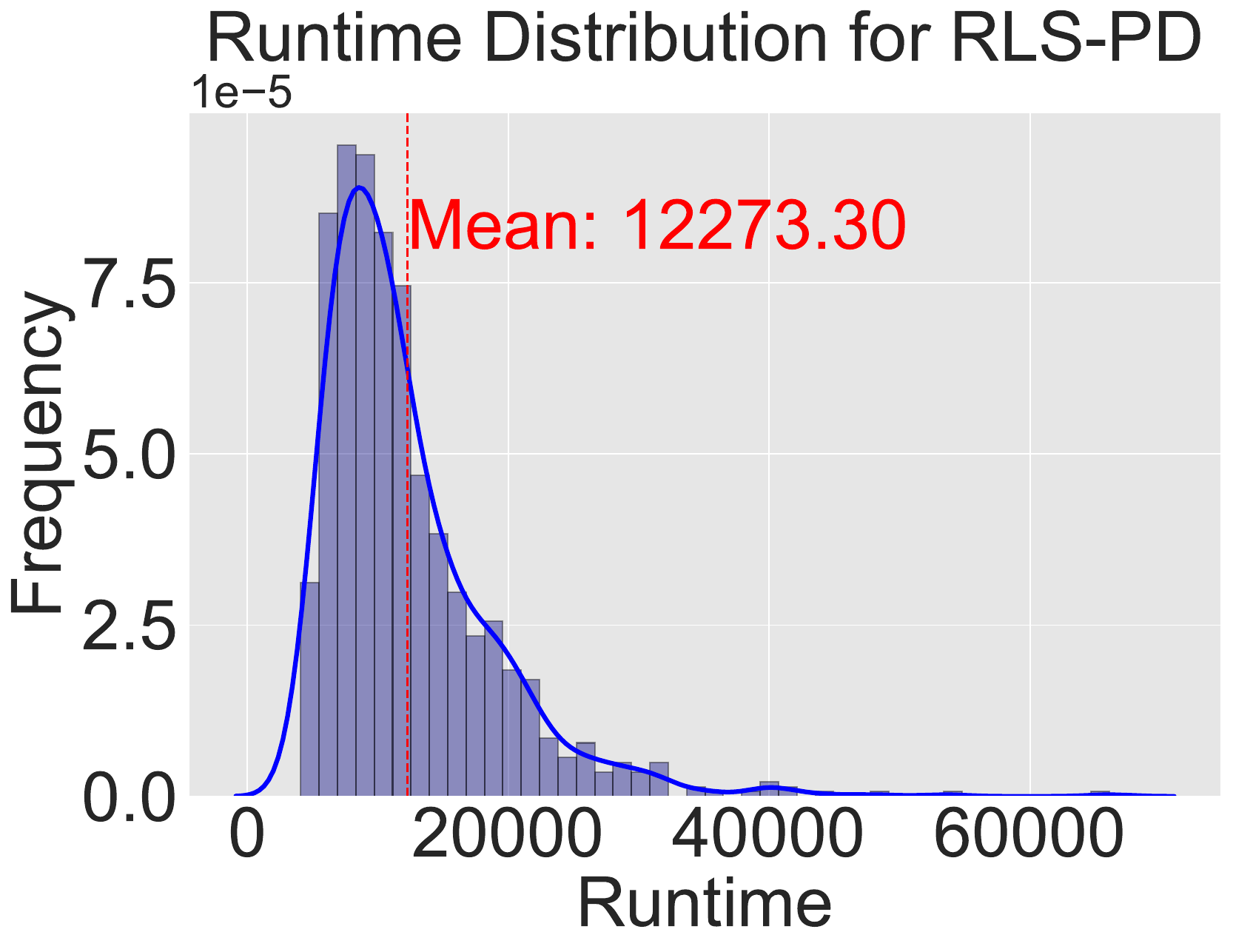}}
    \hfill
  \caption{Runtime distribution for RLS-PD for various values of $\alpha$ and $\beta$, $n=1000$.}
  \label{fig:totalRuntime1} 
\end{figure*}

\begin{figure*}[htbp]
    \centering
  \subfloat[$T=1000, L=5$\label{2a}]{%
       \includegraphics[width=0.33\linewidth]{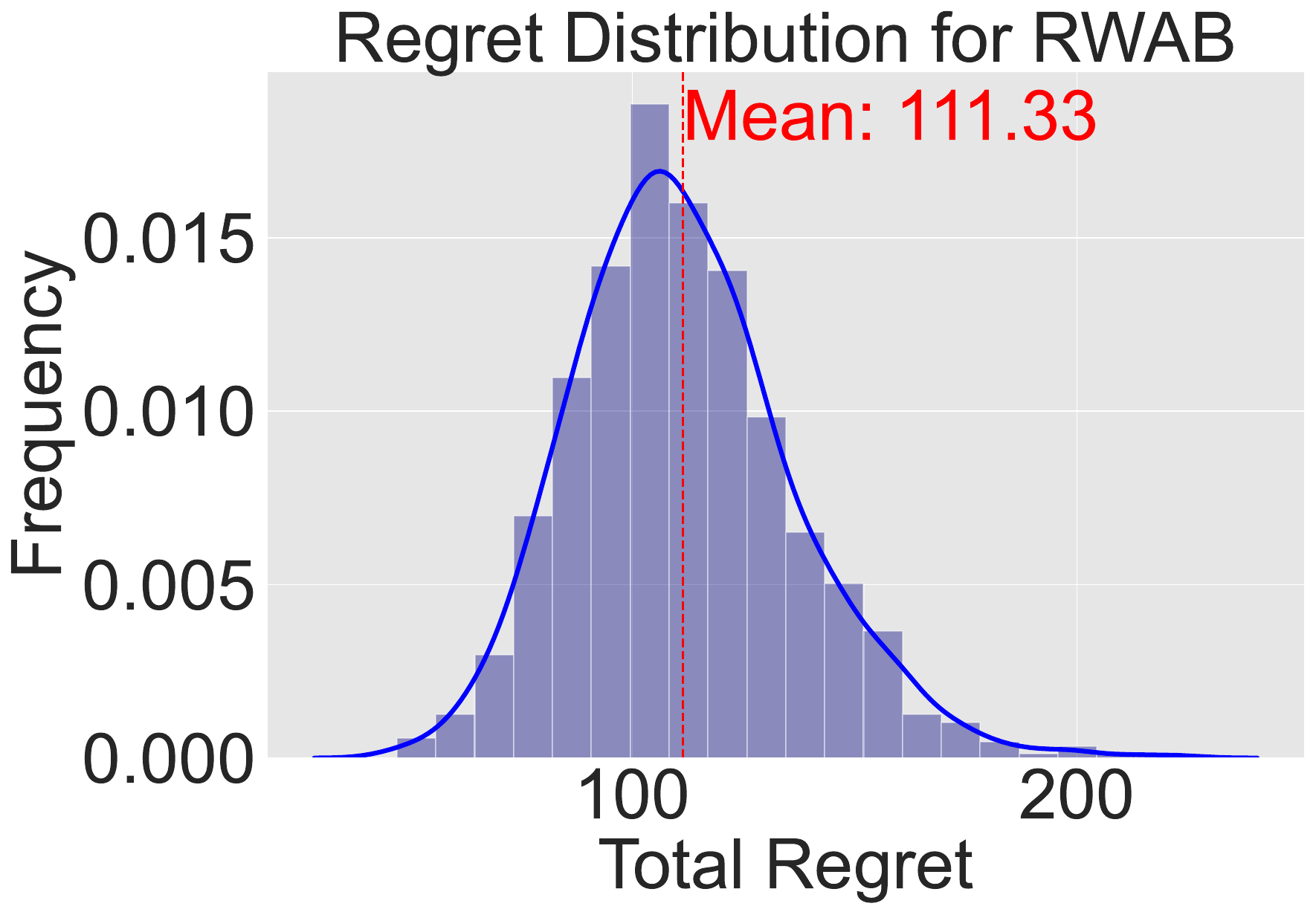}}
    \hfill
  \subfloat[$T=1000, L=10$\label{2b}]{%
        \includegraphics[width=0.33\linewidth]{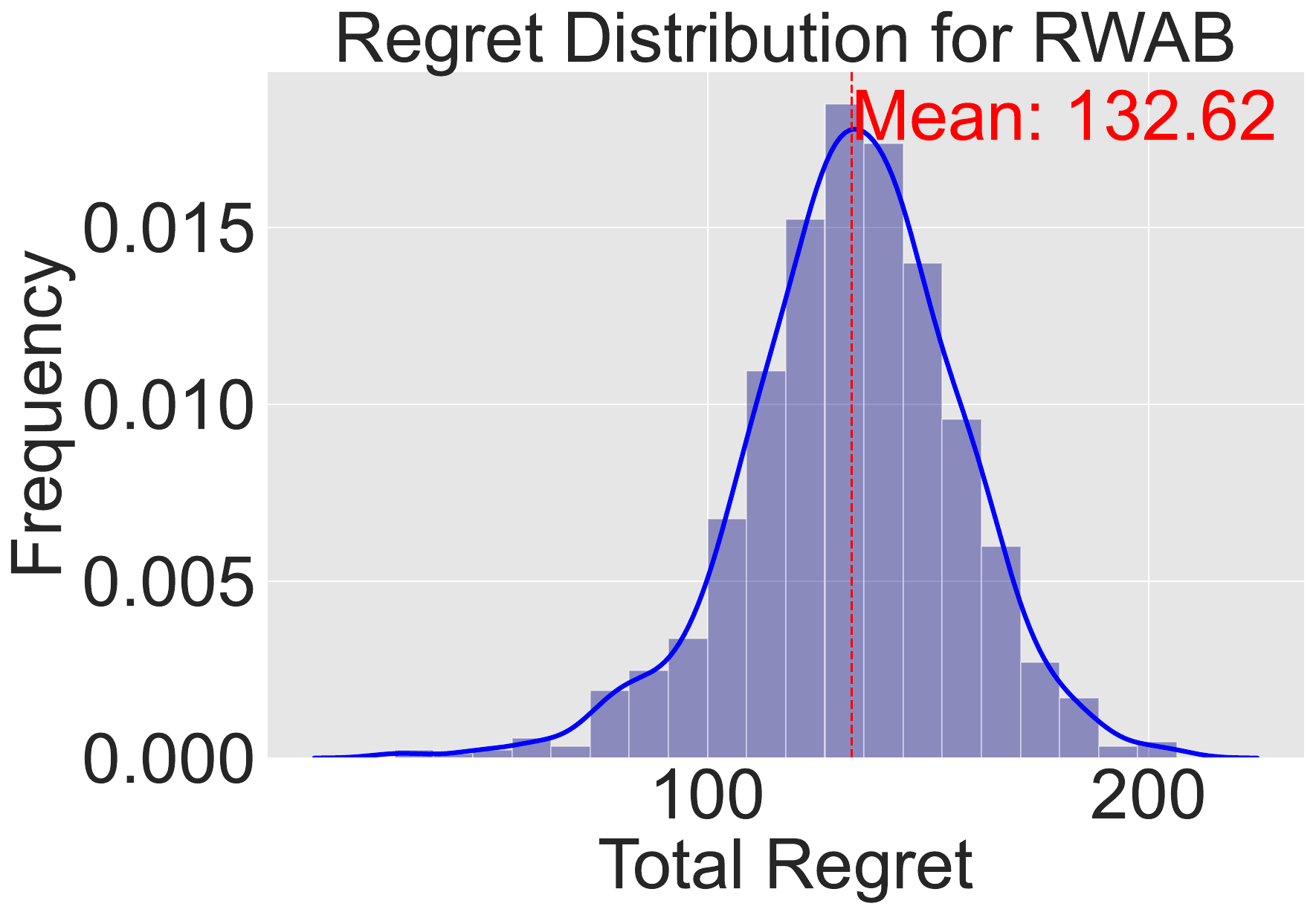}}
  \subfloat[$T=1000, L=20$\label{2c}]{%
        \includegraphics[width=0.33\linewidth]{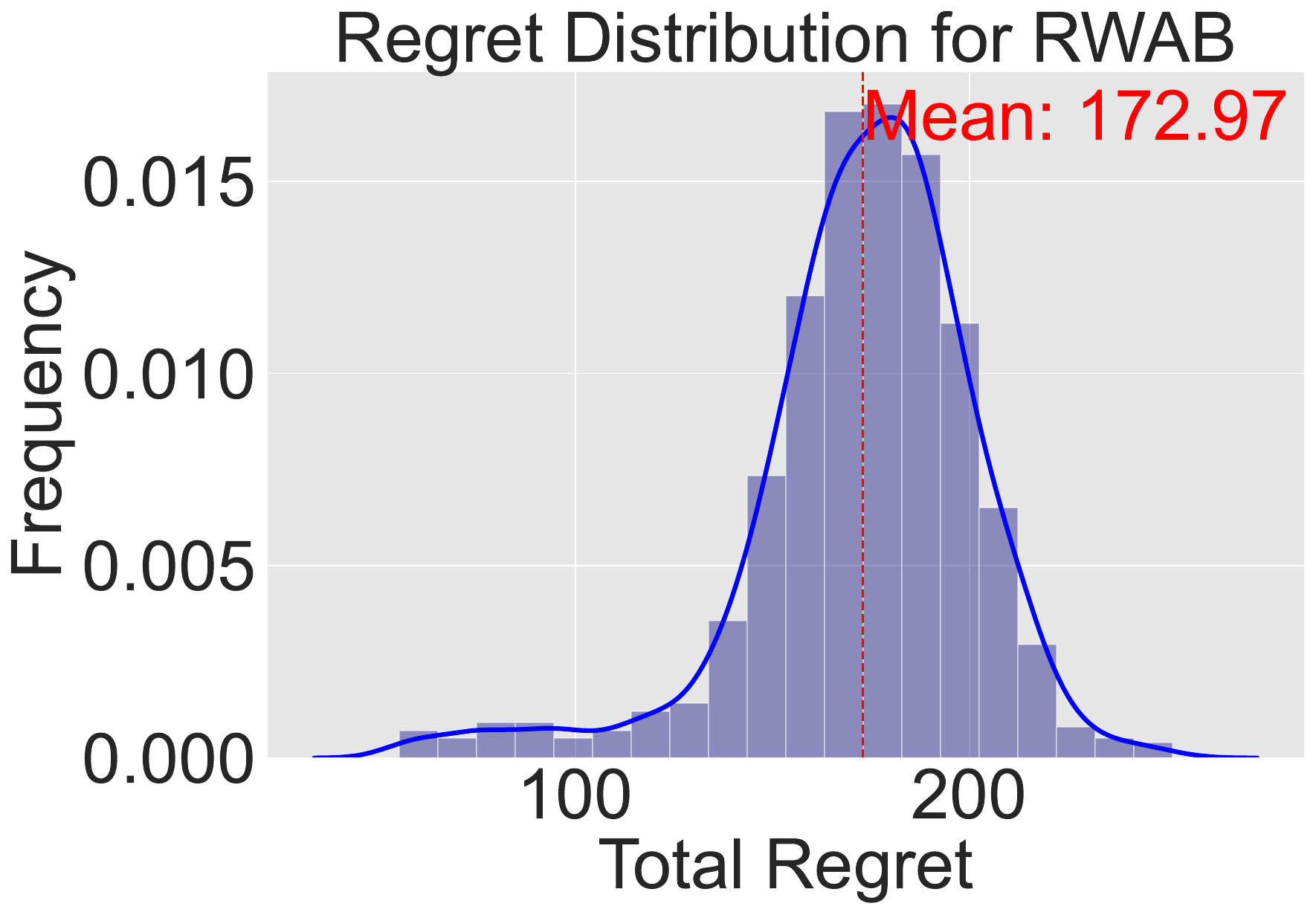}}
    \hfill
  \subfloat[$T=1000, L=40$\label{2d}]{%
        \includegraphics[width=0.33\linewidth]{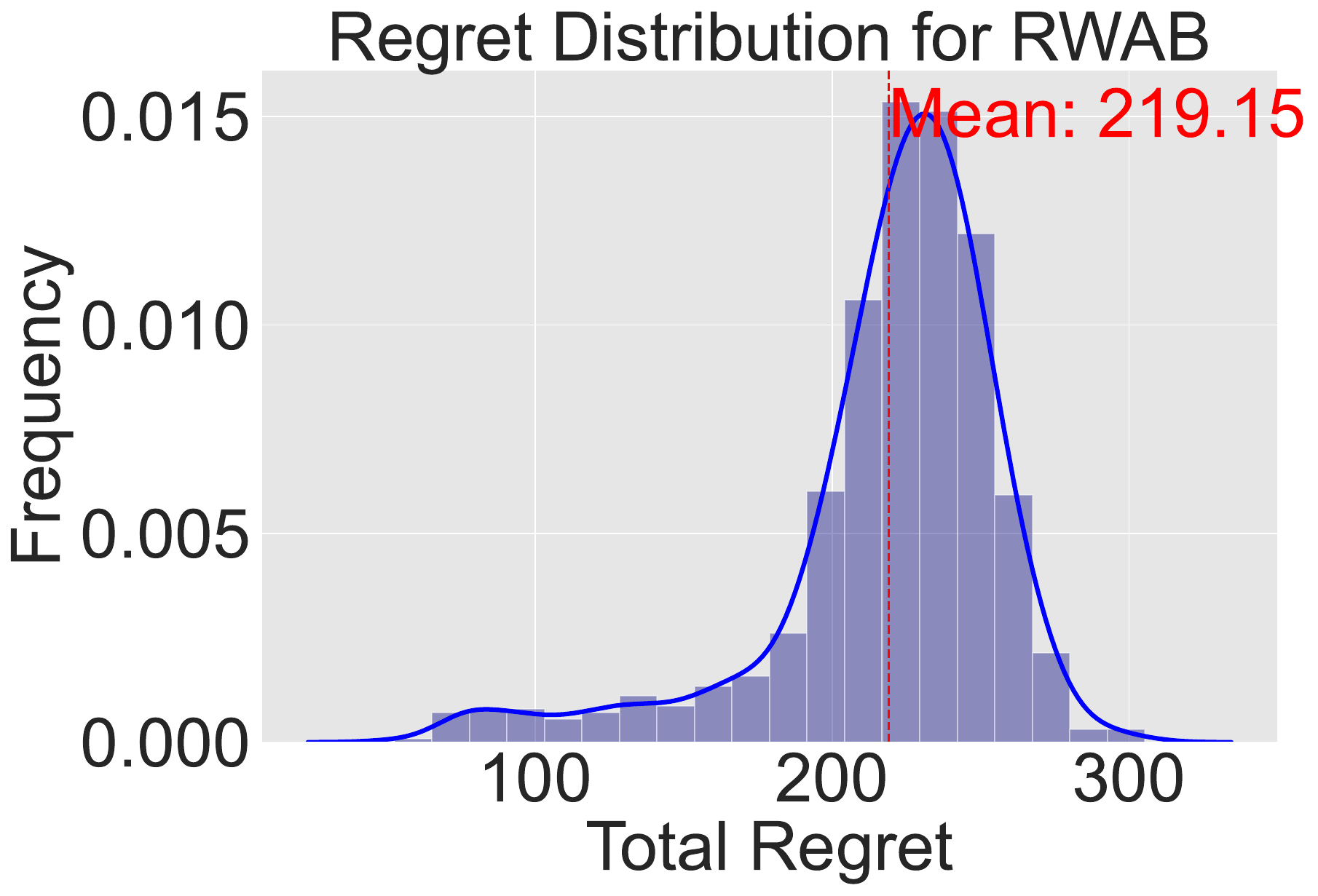}}
    \hfill
  \subfloat[$T=1000, L=80$\label{2e}]{%
        \includegraphics[width=0.33\linewidth]{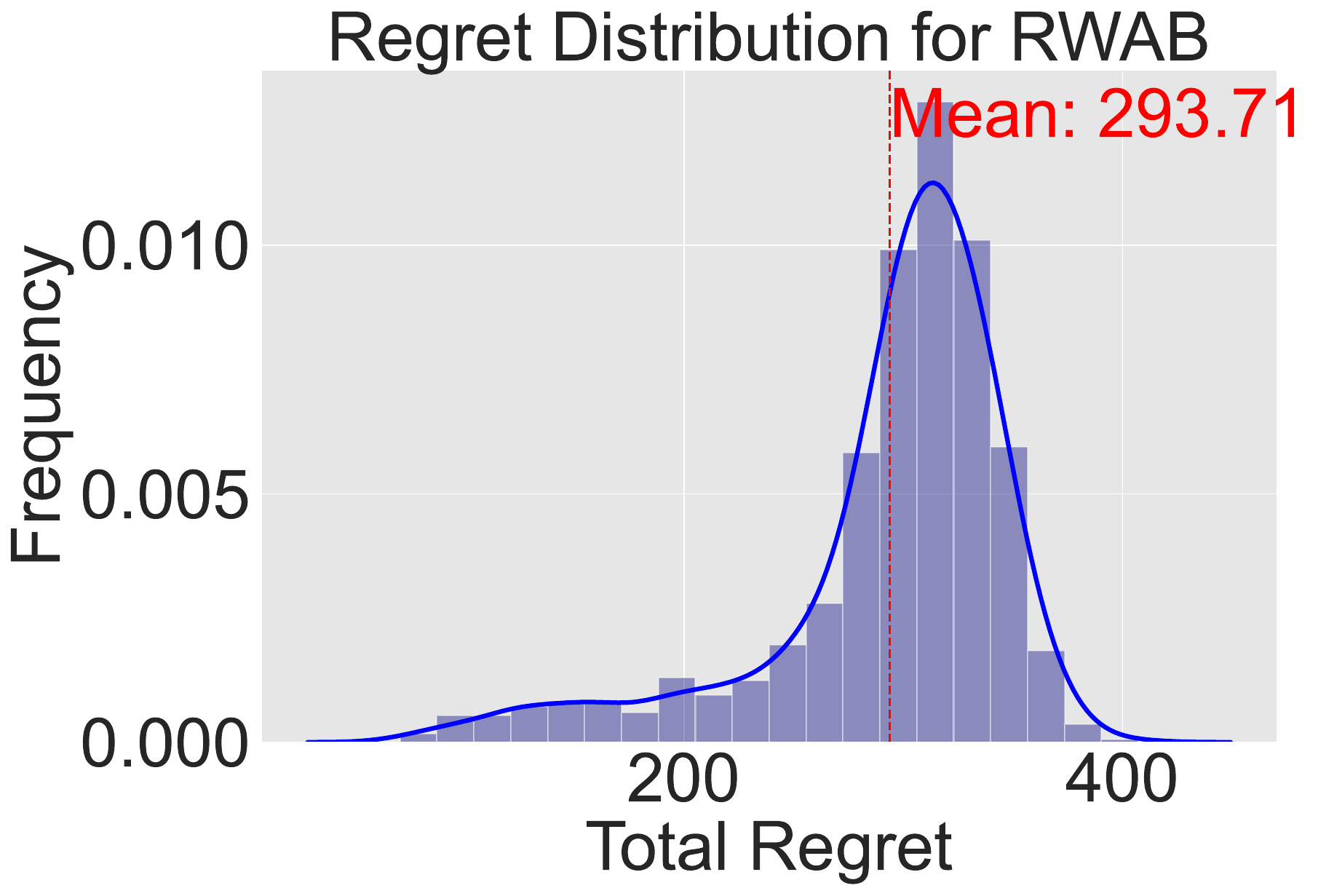}}
    \hfill
  \subfloat[$T=1000, L=100$\label{2f}]{%
        \includegraphics[width=0.33\linewidth]{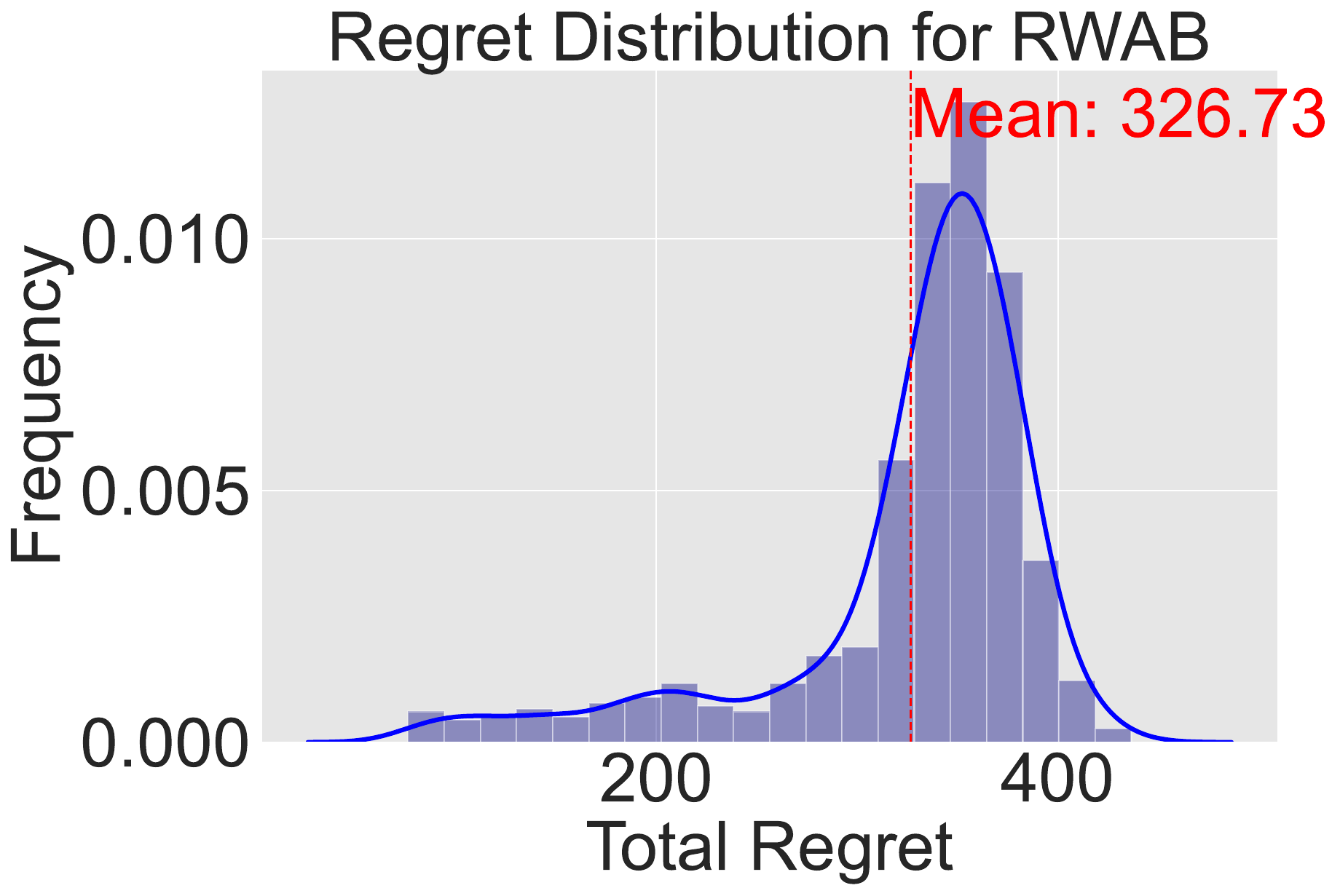}}
  \caption{Regret distribution for various values of $T$ and $L$.}
  \label{fig:totalRegret1} 
\end{figure*}
\begin{table*}[htbp]
    \centering
    \caption{Runtime statistics for RLS-PD on Bilinear $(n=1000)$. $T$ denotes the actual runtime of RLS-PD on Bilinear on each run. $\Bar{T}$ denotes the empirical mean of the runtime and $\text{Fr}$ denotes the frequency of runtimes. }
    \label{tab:runtime_stats}
    \resizebox{\textwidth}{!}{%
    \begin{tabular}{lccccccc}
        \toprule[1.5pt]
        \textbf{Problem Configuration} &
        % Empirical Mean of Runtime 
        \textbf{$\Bar{T}$} &
        \textbf{$\text{Fr} \left(T \leq \Bar{T} \right)$} &
        \textbf{$\text{Fr} \left(T \leq 2\Bar{T}\right)$} &
        \textbf{$\text{Fr}\left(T \leq 4\Bar{T} \right)$} &
        \textbf{$\text{Fr} \left(T \leq 6\Bar{T} \right)$} &
        \textbf{$\text{Fr}\left(T \leq 8\Bar{T}\right)$} \\
        % \textbf{$\Pr \left(T \leq 10E[T]\right)$} \\
        \midrule[0.5pt]
        $\alpha=0.5,\beta=0.5$& \textbf{7029.405} & 0.633 & 0.845 & 0.983 & 0.996 & \textbf{1.0} \\
        $\alpha=0.3,\beta=0.3$ & \textbf{12412.929} & 0.63 & 0.945 & 0.998 & \textbf{1.0} & 1.0  \\
         $\alpha=0.7,\beta=0.7$ & \textbf{12722.888} & 0.604 & 0.95 & 0.997 & \textbf{1.0} & 1.0  \\
       $\alpha=0.3,\beta=0.7$  & \textbf{12738.235} & 0.627 & 0.935 & 0.997 & \textbf{1.0} & 1.0  \\
       $\alpha=0.7,\beta=0.3$  & \textbf{12273.301} & 0.641 & 0.95 & 0.998 & \textbf{1.0} & 1.0  \\
        \bottomrule[1.5pt]
    \end{tabular}%
    }
\end{table*}

\begin{table*}[htbp]
    \centering
    \caption{Regret statistics for Rwab $(T=1000)$. $R$ denotes the actual regret of Rwab on 2-armed non-stationary Bernoulli Bandits on each run. $\Bar{R}$ denotes the empirical mean of the regret and $\text{Fr}$ denotes the frequency of regrets. }
    \label{tab:regret_stats}
    \resizebox{\textwidth}{!}{%
    \begin{tabular}{p{2.2cm}cccccccc}
        \toprule[1.5pt]
        \textbf{Problem}  \newline\textbf{Configuration} &
        \textbf{\(\Bar{R}\)} &
        \textbf{\(\text{Fr}\left(R \leq \Bar{R} \right)\)} &
        \textbf{\(\text{Fr} \left(R \leq 1.2\Bar{R}\right)\)} &
        \textbf{\(\text{Fr} \left(R \leq 1.4\Bar{R} \right)\)} &
        \textbf{\(\text{Fr} \left(R \leq 1.6\Bar{R} \right)\)} &
        \textbf{\(\text{Fr} \left(R \leq 1.8\Bar{R} \right)\)} &
        \textbf{\(\text{Fr} \left(R \leq 2\Bar{R}\right)\)} \\
        \midrule[0.5pt]
        $L=5$ & \textbf{111.329} & 0.543 & 0.829 & 0.949 & 0.990 & 0.998 & \textbf{1.0} \\
        $L=10$ & \textbf{132.621} & 0.491 & 0.871 & 0.988 & \textbf{1.0} & 1.0 & 1.0 \\
        $L=20$ & \textbf{172.975} & 0.457 & 0.920 & 0.996 & \textbf{1.0} & 1.0 & 1.0 \\
        $L=40$ & \textbf{219.151} & 0.383 & 0.949 & \textbf{1.0} & 1.0 & 1.0 & 1.0 \\
        $L=80$ & \textbf{293.709} & 0.343 & 0.945 & \textbf{1.0} & 1.0 & 1.0 & 1.0 \\
        $L=100$ & \textbf{326.728} & 0.299 & 0.958 & \textbf{1.0} & 1.0 & 1.0 & 1.0 \\
        \bottomrule[1.5pt]
    \end{tabular}%
    }
\end{table*}

\end{document}